\documentclass[dvipsnames]{article} %
\usepackage{colm2024_conference}
\pdfoutput=1
\usepackage{booktabs}
\usepackage{graphicx}
\usepackage{enumitem}
\usepackage{wrapfig}
\usepackage{algorithm}
\usepackage{algpseudocode}
\usepackage{microtype}
\usepackage{amsmath}
\usepackage{amsthm}
\usepackage{colortbl}
\usepackage[utf8]{inputenc}
\definecolor{lightgray}{rgb}{0.9,0.9,0.9}
\usepackage{caption}
\usepackage{subcaption}
\usepackage{setspace}
\usepackage{url}
\usepackage{multirow}
\usepackage{colortbl}
\usepackage{tabularx}
\usepackage{blindtext}
\usepackage{pgfplots}
\pgfplotsset{compat=1.18} 
\usepackage{tikz}
\usetikzlibrary{er,positioning,bayesnet}
\usepackage{makecell}
\usepackage{siunitx}
\usepackage{nicefrac}
\usepackage{tocloft}
\usepackage{listings}
\usepackage[nameinlink,capitalize]{cleveref}
\usepackage{placeins}
\usepackage{breakcites}
\usepackage[raster,skins]{tcolorbox} %
\usepackage{xltabular}
\usepackage{adjustbox}
\usepackage{xurl}
\usepackage{rotating}
\usepackage[normalem]{ulem}
\useunder{\uline}{\ul}{}
\usepackage{CJKutf8}
\usepackage{pifont}
\usepackage{xcolor}

\usepackage{amsmath,amsfonts,bm}

\def\eqref#1{equation~\ref{#1}}

\def\1{\bm{1}}

\DeclareMathAlphabet{\mathsfit}{\encodingdefault}{\sfdefault}{m}{sl}
\SetMathAlphabet{\mathsfit}{bold}{\encodingdefault}{\sfdefault}{bx}{n}

\newcommand*\justify{%
  \fontdimen2\font=0.4em%
  \fontdimen3\font=0.2em%
  \fontdimen4\font=0.1em%
  \fontdimen7\font=0.1em%
  \hyphenchar\font=`\-%
}

\renewcommand{\texttt}[1]{%
  \begingroup
  \ttfamily
  \begingroup\lccode`~=`/\lowercase{\endgroup\def~}{/\discretionary{}{}{}}%
  \begingroup\lccode`~=`[\lowercase{\endgroup\def~}{[\discretionary{}{}{}}%
  \begingroup\lccode`~=`.\lowercase{\endgroup\def~}{.\discretionary{}{}{}}%
  \catcode`/=\active\catcode`[=\active\catcode`.=\active
  \justify\scantokens{#1\noexpand}%
  \endgroup
}

\usepackage{amssymb}
\usepackage{fontawesome5}  
\tikzset{capt/.style={}}
\usetikzlibrary{positioning,fit,calc,arrows.meta,backgrounds,decorations.pathreplacing}

\theoremstyle{plain}
\newtheorem{theorem}{Theorem}[section]
\newtheorem{proposition}[theorem]{Proposition}

\theoremstyle{remark}

\definecolor{edge}{HTML}{1B2B3A}
\definecolor{navy}{HTML}{1F3A5F}
\definecolor{sharedgreen}{HTML}{B5E3A2}
\definecolor{fusionpink}{HTML}{E4A2DE}
\definecolor{salmon}{HTML}{F7C7AE}
\definecolor{cellgray}{HTML}{F2F2F2}
\definecolor{cellblue}{HTML}{47AFE0}
\definecolor{barA}{HTML}{45B0E0}
\definecolor{barB}{HTML}{86CCEC}
\definecolor{barC}{HTML}{C4E4F6}
\definecolor{insetgray}{HTML}{C8C8C8}
 
\tikzset{
  every node/.style={font=\sffamily\footnotesize},
  arr/.style={-{Latex[length=4pt,width=3.5pt]}, line width=0.8pt, edge},
  ln/.style={line width=0.8pt, edge},
  block/.style={draw=edge, line width=0.7pt, rounded corners=3pt, minimum height=0.62cm,
                align=center, inner sep=2pt, font=\sffamily\footnotesize},
  shared/.style={block, fill=sharedgreen, minimum width=1.8cm},
  fusion/.style={block, fill=fusionpink},
  decoder/.style={block, fill=salmon, text width=1.8cm},
  tok/.style={draw=edge, line width=0.6pt, fill=cellgray, minimum width=0.64cm, minimum height=0.38cm,
              inner sep=0pt, font=\scriptsize},
  cell/.style={draw=edge, line width=0.6pt, fill=cellgray, rounded corners=2pt,
               minimum width=0.76cm, minimum height=0.42cm, inner sep=0pt},
  bluecell/.style={cell, fill=cellblue},
  emptycell/.style={cell, fill=white, line width=0.45pt, dash pattern=on 1.5pt off 1pt},
  concat/.style={circle, draw=edge, line width=0.7pt, fill=white, minimum size=0.26cm, inner sep=0pt,
                 path picture={\draw[line width=0.8pt, edge]
                   ([xshift=-0.075cm]path picture bounding box.center) -- ([xshift=0.075cm]path picture bounding box.center)
                   ([yshift=-0.075cm]path picture bounding box.center) -- ([yshift=0.075cm]path picture bounding box.center);}},
  lab/.style={font=\sffamily\scriptsize},
  boxdash/.style={draw=navy, line width=0.8pt, dash pattern=on 3.5pt off 2pt},
}

 \newcommand{\stk}[4]{%
  \node[#3] at ({#1+0.2*#2},{3.81+0.3*#2}) {};
  \ifnum#2=0 \node[font=\footnotesize] at (#1,3.81) {#4};%
  \else \node[font=\footnotesize, inner sep=0pt] at ({#1+0.2*#2+0.1},{3.81+0.3*#2+0.07}) {#4};\fi}
 
\newlength{\echoTitleW}
\newcommand{\echoTitleSecond}{%
  \settowidth{\echoTitleW}{\hspace{2.5em}ECHO: Efficient Causal Self-Recursion Beyond}%
  \makebox[\echoTitleW][c]{Auxiliary Multi-Token Prediction}}
\title{ECHO: Efficient Causal Self-Recursion Beyond\\\echoTitleSecond}

\author{%
\begin{tabular}{@{}c@{}}
\bf Yuxuan Zhou$^{1\dagger}$, Junyuan Shang$^{1}$, Tianyu Wang$^{2}$, Chunrui Zheng$^{1}$, Zichang Liu$^{1}$, \\
\bf Wenbin Chang$^{1}$, Shuohuan Wang$^{1}$, Dianhai Yu$^{1}$ \\[0.3em]
$^{1}$\textbf{ERNIE Team, Baidu} \quad $^{2}$\textbf{Independent Researcher} \\[0.2em]
{\small \faGithub\ \href{https://github.com/ernie-research/ECHO}{\texttt{https://github.com/ernie-research/ECHO}}}
\end{tabular}%
}

\usepackage{etoolbox}
\newcommand{\echodate}{Oct 17, 2026}
\patchcmd{\mymaketitle}{Feb 4, 2026}{\echodate}{}{\PackageWarning{echo}{date patch failed}}
\makeatletter
\fancypagestyle{normalpage}{
  \fancyhf{}
  \lhead{\includegraphics[height=15pt]{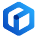}}
  \chead{\textcolor[HTML]{4D4D4D}{\textbf{\thetitle}}}
  \rhead{\textcolor[HTML]{4D4D4D}{\textbf{\echodate}}}
  \renewcommand{\headrulewidth}{0.4pt}
  \renewcommand{\headrule}{{\color[HTML]{4D4D4D}\hrule\@height\headrulewidth\@width\headwidth\vskip-\headrulewidth}}
  \cfoot{\thepage}
}
\makeatother

\makeatletter\def\@abstract{
Existing Multi-token prediction (MTP) methods are auxiliary by design: the multi-horizon loss mostly trains modules attached to the backbone, and each added horizon costs another module and another pass. We propose \textbf{Efficient Causal Self-Recursion (ECHO)}, which removes both costs. The backbone's own final layer predicts every horizon, so the supervision lands on the parameters that generate the next token, and \textbf{Dynamic Depth Sampling (DDS)} keeps deeper horizons affordable by drawing the training depth with probabilities that decay along depth. Future prediction thus becomes part of the autoregressive objective rather than a task attached to it, and the gains follow on both sides of it. In pretraining, added depth converts into downstream accuracy where auxiliary modules saturate, concentrating on reasoning and code at half the MTP training overhead of training every depth. At inference, the pretrained model is its own drafter: it drafts beyond the depth it was trained for, and keeps its draft-target alignment on the open-ended, high-entropy generation where post-hoc drafters lose theirs. The advantage survives a doubled token budget, a repeated pretraining run, and supervised finetuning.

}\makeatother

\let\echoOldSection\section
\renewcommand{\section}{\FloatBarrier\echoOldSection}

\begin{document}
\begingroup\renewcommand\thefootnote{\textdagger}\footnotetext{Project lead.}\endgroup
\maketitle
\renewcommand{\thetitle}{ECHO: Efficient Causal Self-Recursion}

\pagestyle{normalpage}
\clearpage

\section{Introduction}
\label{sec:intro}

\begin{figure}[t]
  \centering
  \begin{subfigure}[t]{3.60in}
    \centering
    \includegraphics[width=\linewidth]{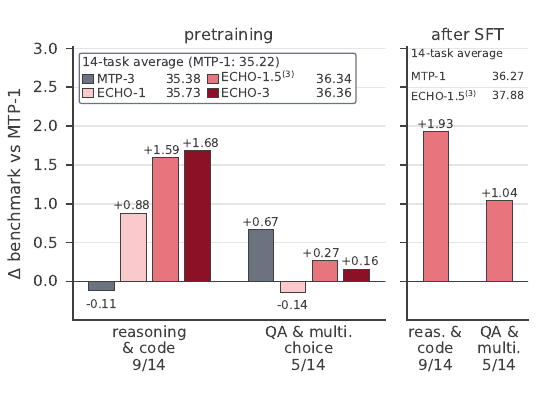}
    \caption{Benchmark gain by task class over the fourteen benchmarks of
      \Cref{tab:benchmark_results,tab:sft}, each panel against the DeepSeek baseline of
      its own stage. ECHO-1.5$^{(3)}$ matches ECHO-3's gain at half its MTP overhead.}
    \label{fig:stage_pairs}
  \end{subfigure}\hfill
  \begin{subfigure}[t]{2.48in}
    \centering
    \includegraphics[width=\linewidth]{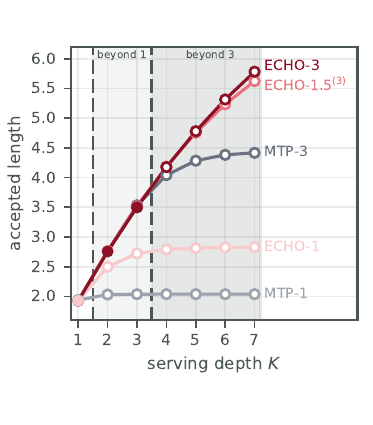}
    \caption{Accepted length against serving depth $K$, filled inside the trained depth
      and hollow beyond it. ECHO keeps gaining where auxiliary MTP saturates.}
    \label{fig:extra}
  \end{subfigure}
  \caption{Multi-horizon supervision on the backbone's own final layer buys downstream performance after pre-training and supervised fine-tuning (a) and drafting depth beyond what was
    trained (b).}
  \label{fig:teaser}
\end{figure}


Large language models (LLMs)
\citep{achiam2023gpt,touvron2023llama,bai2023qwen,liu2024deepseek}
are predominantly trained with the
\emph{Next-Token Prediction} (NTP) objective.
While NTP has demonstrated remarkable scalability and effectiveness,
it provides supervision only at the immediate next-token horizon,
leaving potentially useful future information unexploited.
Moreover, autoregressive decoding remains inherently sequential,
resulting in limited inference throughput.

To better utilize future information,
recent work has explored \emph{Multi-Token Prediction} (MTP)
as an auxiliary objective on top of NTP.
By introducing additional future-token prediction targets,
MTP provides denser supervision during training and can generate draft tokens for speculative decoding, enabling inference acceleration.

Early MTP methods predict multiple future positions independently from the same hidden state \cite{gloeckle2024better}, providing denser supervision but ignoring causal dependencies among future tokens. Subsequent causal MTP methods such as DeepSeek-V3 \cite{liu2024deepseek} address this limitation by introducing autoregressive dependencies among future prediction steps, significantly improving prediction quality and improving the backbone model's overall performance \cite{liu2024deepseek,blakeman2025nvidia,zeng2026glm}. 

However, both causal and parallel MTP remain auxiliary by construction: future-token prediction is carried out by dedicated modules attached to the backbone, rather than being integrated into the autoregressive generation process itself. This leads to two limitations:


\begin{enumerate}

\item \textbf{Indirect Multi-Horizon Supervision}

MTP losses primarily supervise the auxiliary prediction modules.
Consequently, the backbone model receives multi-horizon learning signals only indirectly, potentially limiting their impact on long-horizon planning.

\item \textbf{Costly Long-Horizon Supervision}

Denser supervision requires deeper prediction horizons, yet each additional horizon adds a prediction module in memory and another pass in compute.
This cost limits the prediction horizon in large-scale causal MTP pretraining, typically to a single draft token.

\end{enumerate}

To address the above limitations, we introduce \textbf{Efficient Causal Self-Recursion (ECHO)}, a causal self-recursive
framework that reuses the backbone's final layer for future prediction and
samples per-step prediction depths with decaying probabilities to reduce the amortized MTP training overhead.

\emph{Final-layer parameter sharing} ensures that multi-horizon losses directly shape the backbone, rather than being absorbed by horizon-specific auxiliary modules. 
As shown in \Cref{fig:buckets_ds_pretrain}, this allows ECHO to benefit from deeper supervision: its downstream performance continues to improve at depth 3, whereas DeepSeek's MTP~\citep{liu2024deepseek} does not, consistent with observations in MiMo~\citep{xiaomi2025mimo}. 
More fundamentally, training the same parameters responsible for next-token generation to predict multiple future tokens encourages the backbone to anticipate how the sequence unfolds over subsequent steps, as evidenced by ECHO's substantial gains on reasoning tasks.

To make long-horizon supervision computationally efficient, we introduce \emph{Dynamic Depth Sampling (DDS)}, which samples prediction depths with decaying probabilities. This simultaneously reduces the amortized MTP training overhead and mitigates the impact of increasing prediction errors and supervision noise at greater depths. After supervised finetuning, ECHO's advantage grows further.

\textbf{Beyond its pretraining benefits, ECHO entirely eliminates the need for post-hoc drafter finetuning}. DDS makes longer-horizon supervision affordable, while final-layer parameter sharing enables depth extrapolation: since the same shared layer is reused at every draft step, ECHO can extend drafting beyond the trained prediction depth while maintaining high acceptance rates, at an amortized MTP training overhead of only 1.5 depth-1 steps (see \Cref{sec:draft-deep}). By drafting directly with its pretrained weights, ECHO avoids the overfitting to academic benchmarks associated with post-hoc drafter finetuning~\citep{liu2026angelspecrealworldhighperformance}, yielding substantial acceptance-length gains in realistic, high-entropy generation settings as shown in \Cref{fig:accept_entropy_merged}.

Our main contributions are summarized as follows:

\begin{itemize}

\item We introduce \textbf{Efficient Causal Self-Recursion (ECHO)}, which predicts every future horizon by recursively applying the backbone's own final layer, with the per-step prediction depth drawn from a distribution decaying along depth. 

\item ECHO imposes direct multi-horizon supervision, beating auxiliary MTP on reasoning and code at matched depth while holding on QA \& multiple choice, and widening its lead on both classes after supervised finetuning.

\item ECHO makes long-horizon prediction cheap in compute and free in parameters, and keeps gaining acceptance when extrapolating in depth, where auxiliary MTP saturates. It eliminates drafter finetuning entirely, yet outperforms post-hoc drafters, especially on open-ended generation.
\end{itemize}

\section{Related Work}
\subsection{Multi-Token Prediction}

Multi-Token Prediction (MTP) augments next-token prediction with future-token targets, improving language modeling while providing draft tokens for speculative decoding. It inherits the architecture of speculative draft models and differs in the training regime: draft models are finetuned post-hoc against a frozen backbone on a limited dataset, whereas MTP modules are optimized jointly with the backbone throughout pretraining.

Parallel MTP predicts multiple future positions independently from several output heads on a shared backbone~\cite{gloeckle2024better}, following Medusa~\cite{cai2024medusa}. Causal MTP instead follows EAGLE~\cite{li2024eagle}: DeepSeek-V3~\cite{liu2024deepseek} feeds dedicated prediction modules the previous hidden state together with shifted token embeddings, preserving autoregressive dependencies among future steps, and GLM-5~\cite{zeng2026glm} shares those parameters across all depths, as EAGLE does, to reduce memory cost.

\subsection{Post-hoc Drafter}

Post-hoc drafters are trained on top of a frozen target model by finetuning on target-generated responses, and differ mainly in how a block of draft tokens is produced.
Medusa~\cite{cai2024medusa} attaches parallel heads that draft every position from the final hidden state, while EAGLE-3~\cite{li2025eagle3} re-enters a lightweight decoder once per drafted token, conditioned on target hidden states, preserving token-to-token dependency.
Recent parallel drafters build on DFlash~\cite{chen2026dflashblockdiffusionflash}, which predicts the whole block in one pass, with DSpark~\cite{cheng2026dsparkconfidencescheduledspeculativedecoding} restoring intra-block dependency through a lightweight causal head.

Two limitations are nonetheless inherent to post-hoc drafting.
First, the drafter is narrowly aligned with the target, being finetuned on a dataset orders of magnitude smaller than the pretraining corpus.
Second, scaling that training to the pretraining corpus is compute-heavy from the multi-step or multi-layer overhead, which has not been shown to improve the backbone's downstream performance and may even degrade training.
ECHO overcomes both: joint pretraining aligns it with the target, and its supervision improves the backbone at modest overhead.

\subsection{Looped Transformers}

Looped transformers reuse a shared block of layers across iterations, buying computational depth without new parameters: parameter efficiency~\citep{lanalbert}, adaptive computation~\citep{dehghaniuniversal}, reasoning at LLM scale~\citep{saunshi2025reasoning,zhu2025scaling}, test-time compute scaling~\citep{geiping2025scaling}, and depth allocated per token~\citep{bae2025mixture}. In all of them recursion refines the same token position under a next-token objective; ECHO instead recurses along the horizon, each step conditioning on the next future token to predict one position further ahead.

Yet both buy structure rather than recall, and for looping that is a trade: at matched parameters and compute it leads on reasoning but trails on perplexity and closed-book QA~\citep{saunshi2025reasoning}, and at scale its gain is credited to knowledge manipulation rather than capacity, measured against external dense families rather than a same-recipe baseline~\citep{zhu2025scaling}. ECHO pays no such price: mathematics, reasoning and code gain while multiple choice and closed-book QA hold (\Cref{fig:buckets_ds_pretrain}), and both classes gain after finetuning (\Cref{fig:s5-depth-sft}). Looping, latent chain-of-thought~\citep{hao2024training} and ECHO thus price latent computation differently: repeated passes over one position, annotated chains, and extra positions --- the one price that accepted tokens can repay.

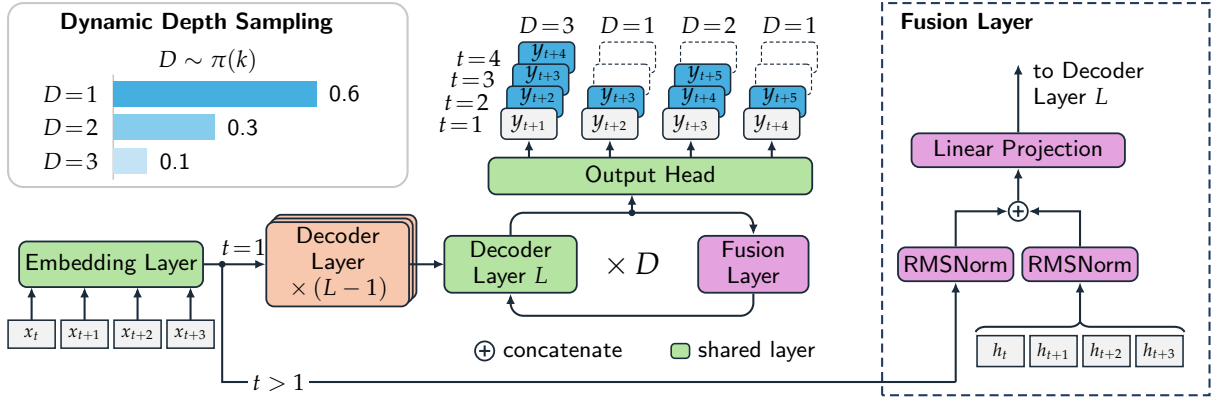
\begin{figure}[t]
  \centering
    \renewcommand{\stk}[4]{%
    \node[#3] at ({#1+0.08*#2},{3.81+0.3*#2}) {};
    \ifnum#2=0 \node[font=\footnotesize] at (#1,3.81) {#4};%
    \else \node[font=\footnotesize, inner sep=0pt] at ({#1+0.08*#2+0.04},{3.81+0.3*#2+0.07}) {#4};\fi}
  \resizebox{\textwidth}{!}{%
\begin{tikzpicture}[x=1cm, y=1cm,
    lab/.style={font=\sffamily\footnotesize}]

\node[shared, minimum width=2.5cm] (emb) at (1.17,1.9) {Embedding Layer};
\foreach \x/\lab in {0.09/x_{t}, 0.81/x_{t+1}, 1.53/x_{t+2}, 2.25/x_{t+3}} {
  \node[tok] at (\x,0.95) {$\lab$};
  \draw[arr] (\x,1.14) -- (\x,1.59);
}
\draw[ln] (emb.east) -- (2.68,1.9);
\fill[edge] (2.68,1.9) circle (1.3pt);
\draw[arr] (2.68,1.9) -- (3.27,1.9);
\node[lab] at (2.975,2.14) {$t\!=\!1$};
\draw[ln] (2.68,1.9) -- (2.68,0.3) -- (12.65,0.3);
\node[lab, fill=white, inner sep=1.5pt] at (3.45,0.3) {$t>1$};

\node[decoder, text opacity=0] at (4.35,2.0) {Decoder Layer\\$\times\,(L-1)$};
\node[decoder, text opacity=0] at (4.3,1.95) {Decoder Layer\\$\times\,(L-1)$};
\node[decoder] (stack) at (4.25,1.9) {Decoder Layer\\$\times\,(L-1)$};
\draw[arr] (stack.east) -- (5.70,1.9);

\node[shared] (decL) at (6.6,1.9) {Decoder\\Layer $L$};
\node[fusion, minimum width=1.5cm] (fus) at (9.9,1.9) {Fusion\\Layer};
\node[font=\large] at (8.25,1.9) {$\times\,D$};
\draw[arr, rounded corners=5pt] (decL.north) -- (6.6,2.60) -- (9.9,2.60) -- (fus.north);
\draw[arr, rounded corners=5pt] (fus.south)  -- (9.9,1.20) -- (6.6,1.20) -- (decL.south);
\fill[edge] (8.25,2.60) circle (1.3pt);
\node[shared, minimum width=4.4cm, minimum height=0.5cm] (head) at (8.50,3.1) {Output Head};
\draw[arr] (8.25,2.60) -- (8.25,2.85);

\foreach \k/\d in {0/3, 1/1, 2/2, 3/1}
  \node[lab] at (6.85+1.1*\k+0.24, 5.15) {$D\!=\!\d$};
\foreach \j in {0,...,3} {
  \pgfmathtruncatemacro{\tnum}{\j+1}
    \node[lab, anchor=east] at (6.35+0.08*\j, 3.81+0.3*\j) {$t\!=\!\tnum$};
}
\stk{6.85}{3}{bluecell}{$y_{\scriptscriptstyle t+4}$}
\stk{7.95}{3}{emptycell}{}
\stk{9.05}{3}{emptycell}{}
\stk{10.15}{3}{emptycell}{}
\stk{6.85}{2}{bluecell}{$y_{\scriptscriptstyle t+3}$}
\stk{7.95}{2}{emptycell}{}
\stk{9.05}{2}{bluecell}{$y_{\scriptscriptstyle t+5}$}
\stk{10.15}{2}{emptycell}{}
\stk{6.85}{1}{bluecell}{$y_{\scriptscriptstyle t+2}$}
\stk{7.95}{1}{bluecell}{$y_{\scriptscriptstyle t+3}$}
\stk{9.05}{1}{bluecell}{$y_{\scriptscriptstyle t+4}$}
\stk{10.15}{1}{bluecell}{$y_{\scriptscriptstyle t+5}$}
\stk{6.85}{0}{cell}{$y_{\scriptscriptstyle t+1}$}
\stk{7.95}{0}{cell}{$y_{\scriptscriptstyle t+2}$}
\stk{9.05}{0}{cell}{$y_{\scriptscriptstyle t+3}$}
\stk{10.15}{0}{cell}{$y_{\scriptscriptstyle t+4}$}
\foreach \x in {6.85,7.95,9.05,10.15}
  \draw[arr] (\x,3.35) -- (\x,3.59);

\node[lab, inner sep=0pt] at (8.425,0.72) {%
  \tikz[baseline=-0.6ex]\node[concat]{};\ concatenate\qquad
  \tikz[baseline=-0.6ex]\draw[edge, line width=0.6pt, fill=sharedgreen, rounded corners=1.5pt] (0,-0.1) rectangle (0.25,0.1);\ shared layer};

\draw[insetgray, line width=0.8pt, rounded corners=6pt] (-0.23,2.95) rectangle (5.20,5.45);
\node[anchor=west, font=\sffamily\bfseries\footnotesize] (ddstitle) at (0.35,5.18)
      {Dynamic Depth Sampling};
\node[lab] at (2.485,4.66) {$D\sim\pi(k)$};
\draw[insetgray, line width=0.6pt] (1.2,3.05) -- (1.2,4.48);
\foreach \y/\d/\p/\c in {4.22/1/0.6/barA, 3.76/2/0.3/barB, 3.30/3/0.1/barC} {
  \node[anchor=east, lab] at (1.12,\y) {$D\!=\!\d$};
  \fill[\c] (1.2,\y-0.17) rectangle ({1.2+4.6*\p},\y+0.17);
  \node[anchor=west, font=\sffamily\footnotesize] at ({1.25+4.6*\p},\y) {\p};
}

\draw[boxdash] (11.65,0.12) rectangle (16.10,5.45);
\node[anchor=west, font=\sffamily\bfseries\footnotesize] at (11.77,5.18) {Fusion Layer};

\node[fusion, minimum width=1.5cm, minimum height=0.46cm] (rmsL) at (12.65,1.9) {RMSNorm};
\node[fusion, minimum width=1.5cm, minimum height=0.46cm] (rmsR) at (14.34,1.9) {RMSNorm};
\draw[arr] (12.65,0.3) -- (rmsL.south);
\foreach \x/\lab in {13.26/h_{t}, 13.98/h_{t+1}, 14.70/h_{t+2}, 15.42/h_{t+3}}
  \node[tok, minimum width=0.64cm, minimum height=0.42cm] at (\x,0.72) {$\lab$};
\draw[ln, decorate, decoration={brace, amplitude=4pt}] (12.92,0.98) -- (15.76,0.98);
\draw[arr] (14.34,1.12) -- (rmsR.south);
\node[concat] (cat) at (13.50,2.62) {};
\draw[arr] (rmsL.north) |- (cat.west);
\draw[arr] (rmsR.north) |- (cat.east);
\node[fusion, minimum width=2.9cm, minimum height=0.46cm] (proj) at (13.50,3.45) {Linear Projection};
\draw[arr] (cat.north) -- (proj.south);
\draw[arr] (proj.north) -- (13.50,4.62);
\node[lab, anchor=west, align=left] at (13.58,4.34) {to Decoder\\Layer $L$};
\end{tikzpicture}
  }
  \caption{Overview of ECHO. The backbone's final decoder layer $L$ and output head are reused to predict every future horizon. Dynamic Depth Sampling (DDS) draws one prediction depth $D\sim\pi$ per optimizer step, with probabilities that decay along depth. Green blocks denote layers shared with the backbone.}
  \label{fig:echo_arch}
\end{figure}

\section{ECHO: Efficient Causal Self-Recursion}
We first show that final-layer sharing benefits the backbone through direct multi-horizon supervision, and then dynamic depth sampling makes it affordable at greater depth.
\subsection{Direct Multi-Horizon Supervision via Final-Layer Sharing}
We show that reusing the backbone's final layer turns multi-horizon supervision into a direct training signal: downstream performance improves over auxiliary MTP, and shows further gains at greater depth.

\paragraph{Causal Self-Recursion:}
we follow the autoregressive design of DeepSeek-V3~\citep{liu2024deepseek} to preserve causal dependency, where each additional horizon is predicted from the previous final-layer hidden state $h^L_{t-1}$ and the embedding $e(y_{t})$ of the previously predicted token. As shown in~\Cref{fig:echo_arch},
rather than attaching a dedicated module for each depth, we recursively apply the backbone's final decoder layer $T^{L}(\cdot)$ and output head at every horizon, connected to the previous step by a lightweight fusion layer $F(\cdot)$:
\begin{equation}
\label{eq: hidden-states}
    h^L_{t} =
    \begin{cases} 
        h^L_1, & t=1, \\
        T^L\big(F\left(h^L_{t-1},e(x_{t})\right)\big), & t>1,
    \end{cases}
\end{equation}
apart from the lightweight fusion layer, every parameter used for future prediction is also used for next-token generation. Supervision from every horizon shapes the backbone directly, and the same layer can be applied repeatedly to draft beyond the trained depth.

We show two benefits of the resulting direct multi-horizon supervision:

\paragraph{Sharing improves downstream performance, especially in reasoning capability.}
To isolate the effect of sharing, we compare ECHO with DeepSeek's MTP~\citep{liu2024deepseek} at matched prediction depth (\Cref{fig:buckets_ds_pretrain}).
Under this comparison, sharing the final layer gives ECHO the lead at depth~1, most clearly on math, reasoning and code ($+0.79$, $+0.57$ and $+1.35$), and it reaches $+2.63$ on reasoning and $+2.23$ on code at depth~3.
The difference lies in where multi-horizon supervision mainly lands.
With auxiliary modules, the signal is largely absorbed by the modules dedicated to individual horizons, leaving the backbone's hidden layers specialized for the next token~\citep{mehra2025multitoken}.
Sharing removes this bottleneck by construction: since the final layer predicts every horizon, supervision from all horizons is concentrated on the parameters that generate the next token.

\paragraph{Sharing promotes longer-horizon supervision.}
To examine how each design uses added depth, we compare each family with itself from depth~1 to depth~3.
Averaged over all fourteen benchmarks, deeper supervision improves ECHO by $0.63$ points, whereas DeepSeek's MTP gains only $0.16$, and ECHO's advantage on reasoning and code persists as depth increases (\Cref{fig:buckets_ds_pretrain}).
This divergence follows from where each new horizon's loss goes: in ECHO it reaches the backbone directly, whereas in DeepSeek's MTP it mostly trains the newly added module, as discussed above.
This explains why auxiliary MTP rarely goes deeper in practice: DeepSeek-V3 trains with a single MTP module~\citep{liu2024deepseek}, and MiMo finds that adding MTP layers brings no further improvement~\citep{xiaomi2025mimo}.

\begin{figure}[t]
  \centering
  \makebox[\textwidth][s]{\hfill%
  \begin{subfigure}[t]{0.49\textwidth}
    \centering
    \includegraphics[width=\linewidth]{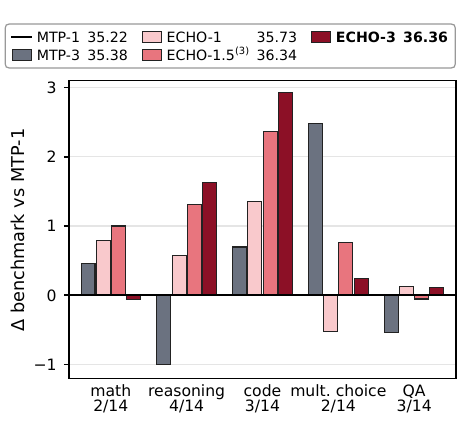}
    \caption{In pretraining, the gain concentrates on reasoning and code, and grows with depth.}
    \label{fig:buckets_ds_pretrain}
  \end{subfigure}\hfill%
  \begin{subfigure}[t]{0.49\textwidth}
    \centering
    \includegraphics[width=\linewidth]{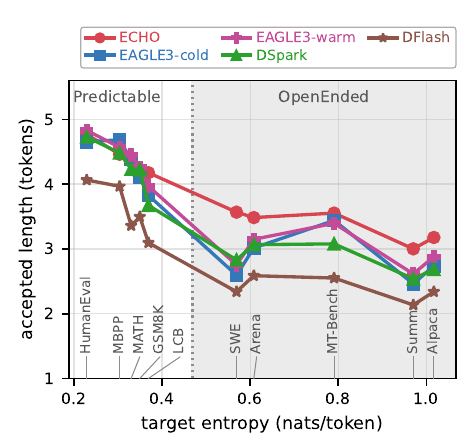}
    \caption{At inference, it keeps the longest accepted length, widening with entropy.}
    \label{fig:accept_entropy_merged}
  \end{subfigure}\hfill}%
  \caption{\textbf{What final-layer sharing buys, in pretraining and at inference.}
    (a) Mean per-task score difference from MTP-1 within each task category, over the
    fourteen benchmarks of \Cref{tab:benchmark_results}, with the category size below
    each label and the fourteen-task averages in the legend.
    (b) Accepted length at $T{=}1$ with rejection sampling, against the target's
    predictive entropy; the warm-start EAGLE-3 drafter is trained from DeepSeek's MTP
    baseline, following \citet{kimiteam2026k3}.}
  \label{fig:sharing_pays}
\end{figure}

\FloatBarrier

\subsection{Affordable Long-Horizon Supervision via Dynamic Depth Sampling (DDS)}
\label{sec:dds}

We show that dynamically sampling the training depth lowers the cost of long-horizon supervision without sacrificing its gains in downstream performance and drafting quality.

\paragraph{Dynamic Depth Sampling.}
At each optimizer step, DDS draws a prediction depth $D \sim \pi(k)$ over $\{1, \dots, K\}$ and unrolls the recursion of~\eqref{eq: hidden-states} for depths $1$ through $D$ (\Cref{fig:echo_arch}), so the MTP training overhead scales with the expected depth
\begin{equation}
  \label{eq:expected-depth}
  \mathbb{E}[D] = \sum_{k=1}^{K} k\,\pi(k) .
\end{equation}
In practice, we use $\pi = (0.6, 0.3, 0.1)$ with $K{=}3$, giving $\mathbb{E}[D] = 1.5$, half the overhead of training all three depths.
Throughout, this \emph{MTP training overhead} is reported in units of the depth-1 module, that is, unrolls of the shared layer per optimizer step; it is the increment on top of the backbone, not the end-to-end cost of a training step.

\paragraph{DDS lowers the MTP training overhead while preserving downstream benefit.}
ECHO-1.5$^{(3)}$ matches ECHO-3 on the benchmark average ($36.34$ vs.\ $36.36$, \Cref{fig:buckets_ds_pretrain}) at half the MTP training overhead.
This holds because DDS changes how often each depth is trained, not which depths are trained.
Since depths are sampled from $\{1,2,3\}$, ECHO-1.5$^{(3)}$ still receives depth-3 supervision, only not at every step, so it keeps the benefit of long-horizon supervision without unrolling every depth in every batch.
Because deeper horizons condition on longer chains of recursive hidden states, their supervision grows noisier with depth, so DDS concentrates training on reliable shallow positions and protects the shared backbone from deep-horizon noise, much as stochastic depth trains deeper layers less often~\citep{huang2016deep}.
At this cost, ECHO-1.5$^{(3)}$ also outperforms MTP-1 on math, reasoning and code ($+1.00$, $+1.31$, $+2.37$) and still leads MTP-3 at twice its budget ($+0.55$, $+2.31$, $+1.67$; \Cref{fig:buckets_ds_pretrain,tab:benchmark_results}).

\section{Pushing the Frontier of Speculative Decoding under Realistic Serving}
\label{sec:serving}
We analyze how ECHO's end-to-end paradigm intrinsically aligns the drafter with the target, entirely eliminating the need for post-hoc drafter finetuning.

\subsection{Training Shallow, Drafting Deep}
\label{sec:draft-deep}

We first show how DDS aligns training with the acceptance dynamics of speculative decoding, then demonstrate how last-layer sharing enables depth extrapolation beyond the trained depth.

\paragraph{DDS aligns training with speculative decoding.}
At inference, the drafter proposes $K$ tokens and the target verifies them in parallel, but draft tokens are committed only up to the first rejection~\citep{leviathan2023fast}.
With acceptance rate $a_i \in (0,1]$ at position $i$, position $d$ is reached with probability $\prod_{i=1}^{d-1} a_i$~\citep{wu2026dpace,zhou2026overcomingjointintractabilitylossless}, so the expected accepted length is
\begin{equation}
  \label{eq:depth-match}
  \mathbb{E}[\tau] = 1 + \sum_{d=1}^{K} \prod_{i=1}^{d} a_i ,
\end{equation}
where $\tau$ is the accepted length, the bonus token plus the accepted drafts.
A rejection at depth $k$ thus truncates the $c_k = K - k + 1$ drafts at and after it, so shallow depths dominate accepted length.
DDS effectively trains depth $k$ with a frequency that decays along depth,
\begin{equation}
  \label{eq:dds-weight}
  f_k = \mathbb{P}(D \ge k) = \sum_{j=k}^{K} \pi(j) ,
\end{equation}
so $\pi = (0.6, 0.3, 0.1)$ gives $f = (1, 0.4, 0.1)$.
Under lossless verification, the rejection rate is $1 - a_k = \mathrm{TV}(p_k, q_k)$~\citep{leviathan2023fast}.
LK losses~\citep{samarin2026lklosses} exploit this identity to optimize acceptance directly, but treat each position equally.
Instead, relating the rejection rate to the data distribution $d_k$ and accounting for truncation across depths gives (\cref{app:dds})
\begin{equation}
  \label{eq:chain}
  \mathbb{E}[\tau]
  \ \ge\ 1 + K - \tfrac{1}{\sqrt{2}}\, \mathcal{L}_{\text{MTP}}(f)
  \ \ge\ 1 + K - \tfrac{1}{\sqrt{2}}\, \mathcal{L}_{\text{MTP}}(\mathbf{1}),
\end{equation}
where $\mathcal{L}_{\text{MTP}}(f) = \sqrt{\sum_k c_k^2 / f_k}\,\sqrt{\sum_k f_k \mathrm{KL}(d_k \,\Vert\, q_k)}$ bounds the MTP loss.
The last inequality compares DDS with the uniform depth supervision  ($f=\mathbf{1}$) and holds for schedules that decay toward the optimum (\cref{prop:decay}).
The uniform depth supervision weights every depth's error equally, whereas decaying $f_k$ tightens the bound, so reducing the training loss translates more directly into accepted length.
DDS thus concentrates target-draft alignment on the shallow depths that determine acceptance, preserving accepted length despite reduced supervision at greater depths.
Existing drafters encode this decay by reweighting the loss while still training every depth at every step~\citep{chen2026dflashblockdiffusionflash,cheng2026dsparkconfidencescheduledspeculativedecoding};
DDS instead encodes it in how often each depth is trained, so the training compute follows the decay.

\paragraph{Extrapolating Beyond the Training Horizon.}
We refer to drafting beyond the maximum depth seen during training as \emph{depth extrapolation}. As shown in~\Cref{fig:extra}, beyond the trained depth ($K>3$), ECHO keeps gaining at a nearly constant rate, and ECHO-1.5$^{(3)}$ retains this gain at half the MTP training overhead of ECHO-3, whereas the gains of MTP-3 shrink with each step and saturate. Sharing also makes extrapolation the natural complement to DDS: training concentrates supervision on shallow depths, and recursion carries it to deeper ones at inference.

\subsection{Staying Aligned under realistic setting}

We then show that ECHO outperforms existing post-hoc drafters, with larger gains on high-entropy, realistic, and open-ended tasks, thanks to two complementary factors: causal drafting ensures consistency among draft tokens, while joint pretraining aligns the draft distribution with the target.

\begin{figure}[t]
  \centering
  \begin{subfigure}[b]{0.33\textwidth}
    \centering
    \includegraphics[width=\linewidth]{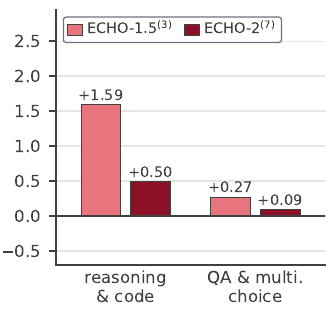}
    \caption{Deeper sampling trails ECHO-1.5$^{(3)}$ before finetuning.}
    \label{fig:s5-depth-pre}
  \end{subfigure}\hfill
  \begin{subfigure}[b]{0.33\textwidth}
    \centering
    \includegraphics[width=\linewidth]{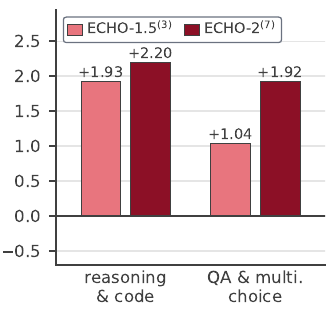}
    \caption{Finetuning widens the lead on both classes.}
    \label{fig:s5-depth-sft}
  \end{subfigure}\hfill
  \begin{subfigure}[b]{0.33\textwidth}
    \centering
    \includegraphics[width=\linewidth]{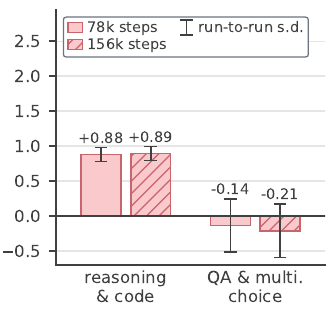}
    \caption{The lead against the spread two matched runs measure.}
    \label{fig:s5-conditions}
  \end{subfigure}
  \caption{\textbf{The advantage persists under deeper sampling, supervised finetuning
    and a doubled token budget.} Mean
    per-task gain over the matched DeepSeek baseline, by task class: MTP-1 before
    finetuning, MTP-1-Instruct after it. (\subref{fig:s5-conditions}) shows the gap at
    step 78k and at step 156k, and the error bars are the run-to-run standard deviation
    from an independent repeat of the 78k recipe, $0.08$ on reasoning \& code and
    $0.38$ on QA \& multiple choice. Doubling the budget moves
    the reasoning \& code gap by $0.01$, a sixth of that spread, whereas the repeat
    moves the QA \& multiple choice gap by $0.54$, more than its own. The reasoning \& code
    advantage therefore survives both conditions and the run-to-run spread, while the
    QA \& multiple choice difference stays inside it.}
  \label{fig:section5}
\end{figure}
\paragraph{Causal Drafting Matters in Open-Ended Generation.}
As shown in~\Cref{fig:accept_entropy_merged}, ECHO and other causal drafters outperform parallel drafters on high-entropy tasks, consistent with~\citet{liu2026angelspecrealworldhighperformance}; \Cref{fig:entropy} shows the same ordering in decode speedup.
When the next tokens are uncertain, the target's choice is hard to anticipate, so the probability of acceptance falls off steeply with depth and most of a long draft is discarded.
A causal drafter adapts to this by extending its draft step by step, while a parallel drafter emits a fixed-length block whose deeper tokens are mostly rejected at verification.
The balance shifts on simple code and mathematics benchmarks such as HumanEval, MBPP, and GSM8K, where long stretches of output are determined by syntax or derivation, and block-parallel drafting becomes worthwhile~\citep{liu2026angelspecrealworldhighperformance}, whereas realistic software engineering tasks such as SWE-bench remain high-entropy and still favor causal drafting (\Cref{fig:accept_entropy_merged}).
Such simple tasks, however, are a small fraction of real-world usage, account for under 8\% of consumer ChatGPT messages, while open-ended practical guidance, information seeking, and writing account for nearly 80\%~\citep{chatterji2025people}.

\paragraph{Joint Pretraining Closes the Alignment Gap.}
Acceptance depends on how closely the draft distribution matches the target's~\citep{leviathan2023fast,zhou2024distillspec}, and a drafter can better match the target where it has seen the target's data. ECHO is trained with the target end-to-end and sees its full pretraining corpus. In contrast, auxiliary drafters are fit to a frozen target on a small finetuning set, so their acceptance drops under distribution shift~\citep{liu2024online,hong2025training}.
The EAGLE-3 pair in~\Cref{fig:accept_entropy_merged} isolates this effect: with architecture and finetuning data held fixed, warm-started EAGLE-3 achieves a longer accepted length than its cold-start counterpart, consistent with recent findings that pretrained initialization strengthens drafters, most on out-of-domain inputs~\citep{bie2026osprey}.
Yet on open-ended tasks, warm-starting still falls short of ECHO, with the largest gap on SWE-bench.
ECHO is instead optimized jointly with the target on every pretraining update, so its alignment is maintained throughout training, in line with the high acceptance rates reported for MTP modules trained jointly during pretraining~\citep{liu2024deepseek,xiaomi2025mimo}.

\section{The Advantage Holds across Data, Depth and Runs}
\label{sec:persistence}

We finally examine  whether ECHO's advantage holds along three axes: more pretraining data, deeper sampling, and repeated runs.

\subsection{Scaling with Data}

\paragraph{ECHO's advantage persists with more pretraining data.}
We extend pretraining of ECHO-1 and DeepSeek's MTP-1 from step 78k to step 156k, doubling the token budget.
Both models improve, and ECHO-1 keeps its lead over DeepSeek's MTP-1 measured at the same step: $+0.52$ at step 78k and $+0.50$ at step 156k, and on reasoning \& code it is $+0.88$ and $+0.89$ (\Cref{fig:s5-conditions}, \Cref{tab:train-twice}).
This indicates that ECHO's advantage is stable under data scaling rather than an early-training effect that more data would erase.

\subsection{Scaling with Depth}
\label{sec:persist-depth}

\paragraph{ECHO's advantage holds across sampling depths.}
We next train ECHO-2$^{(7)}$, which samples $D$ over $\{1,\dots,7\}$ with $\mathbb{E}[D]=2$, at $4/3$ the MTP training overhead of ECHO-1.5$^{(3)}$.
It trails ECHO-1.5$^{(3)}$ in pretraining but overtakes it after SFT by 0.49, reaching the largest lead over DeepSeek's MTP-Instruct of any configuration ($+2.10$; \Cref{tab:sft}).
The gap in pretraining may be due to errors accumulating at deeper positions, which makes their supervision noisier and draws training away from the reliable shallow depths.
After SFT, however, the deeper supervision appears to pay off, possibly because finetuning makes better use of what the backbone learned.

\subsection{Stability across Runs}
\label{sec:persist-variance}

\paragraph{ECHO's lead is robust to training and evaluation noise.}
A reported gap can come from noise rather than the method itself.
We eliminate evaluation noise by decoding greedily, so the remaining source is training.
To measure it, we pretrain ECHO-1 and DeepSeek's MTP-1 a second time under the same recipe and compare the two runs at the same step.
Across the two runs, ECHO leads DeepSeek's MTP-1 by $+0.52$ and $+0.78$, and its lead on reasoning \& code shifts by only $0.11$ (\Cref{fig:s5-conditions}; \Cref{app:stability}).
ECHO's overall and reasoning \& code leads are therefore reproducible across runs.

\section{Experiment Setup}
\label{sec:experiment-setup}

Full architecture, prompt, scoring and speculative decoding details are given in \Cref{sec:appendix-setup}.

\noindent \textbf{Models and Baselines}: Every configuration shares a 10B-A1B Latent MoE backbone \citep{blakeman2025nvidia} with 24 layers, multi-head latent attention \citep{liu2024deepseek}, and 256 routed experts with eight active per token \citep{dai2024deepseekmoe}, trained under the same recipe so that the comparison isolates the MTP design.

\noindent \textbf{Pretraining}: 600B tokens at sequence length 4,096 and batch size 1,920 in bfloat16 with Muon \citep{jordan2024muon}, on 24 nodes of eight H800 GPUs.

\noindent \textbf{Dynamic Depth Sampling}: ECHO-1.5$^{(3)}$ samples $D$ over $\{1,2,3\}$ with $\pi=(0.6,0.3,0.1)$, giving $\mathbb{E}[D]=1.5$ and ECHO-2$^{(7)}$ samples over $\{1,\dots,7\}$ with $\pi=(0.520,0.220,0.121,0.068,0.036,0.019,0.016)$, giving $\mathbb{E}[D]=2$.

\noindent \textbf{Supervised Finetuning}: Every configuration is finetuned with the same recipe, 5,000 steps over one pass of an in-house chat corpus at context length 32,768, keeping its pretrained MTP depths.

\noindent \textbf{Computational Cost}:
Each pretraining run takes approximately 3 days on 192 NVIDIA H800 GPUs with a global batch size of 1,920, followed by 5 hours of SFT on the same hardware. Post-hoc drafter training uses 24 H800 GPUs with a global batch size of 48, requiring about 22 hours for EAGLE-3, 19 hours for DFlash, and 55 hours for DSpark.

\noindent \textbf{Benchmark Evaluation}: Fourteen few-shot \citep{brown2020language} benchmarks decoded greedily, grouped by task form into mathematics, reasoning, multi-subject multiple choice, code and closed-book question answering, with CoT prompting \citep{wei2022chain} on MATH, GSM8K and BBH.

\noindent\textbf{Speculative decoding:} We adopt the same temperature setting ($T=1$) for both the target and draft models, following Bebop~\citep{li2026breakingentropyboundsaccelerating} and GLM-5.2~\citep{glm52}.  This matters for RL training: rejection sampling with draft temperature$>0$ keeps acceptance stable as target entropy rises, whereas greedy drafting degrades sharply at higher temperatures~\citep{li2026breakingentropyboundsaccelerating}. All drafters are served with SGLang on a single node of eight H800 GPUs and budget-matched at seven proposed positions plus one bonus token per verification step. We report decode-only throughput and mean accepted length on ten tasks, split by the target model's predictive entropy into predictable and open-ended groups.

\section{Conclusion}

We presented ECHO, which turns multi-token prediction from an auxiliary task into part of the autoregressive objective.
By reusing the backbone's own final layer at every horizon, ECHO lets multi-horizon supervision shape the parameters that generate the next token, and Dynamic Depth Sampling makes deeper horizons affordable by training them less often.
Unlike auxiliary MTP such as DeepSeek, ECHO improves with added depth, and ECHO-1.5$^{(3)}$ matches the downstream gain of ECHO-3 at half the MTP overhead.

At inference, the same shared layer makes the pretrained model its own drafter with close draft-target alignment: it drafts beyond its trained depth where auxiliary modules saturate, and it outperforms post-hoc drafters on open-ended, high-entropy generation.
This advantage holds under a doubled token budget, a repeated pretraining run, and supervised finetuning.

\section*{Limitations}
Due to compute constraints, pretraining was repeated once and finetuning was run with a single seed.
How far the gains from added depth and from depth extrapolation extend, how DDS behaves at greater depths, and why deeper sampling pays off only after finetuning are left to future work.


\clearpage
\bibliography{Main}

@article{achiam2023gpt,
  title={{GPT-4} Technical Report},
  author={Achiam, Josh and Adler, Steven and Agarwal, Sandhini and Ahmad, Lama and Akkaya, Ilge and Aleman, Florencia Leoni and Almeida, Diogo and Altenschmidt, Janko and Altman, Sam and Anadkat, Shyamal and others},
  journal={arXiv preprint arXiv:2303.08774},
  year={2023}
}

@inproceedings{hao2024training,
  title     = {Training Large Language Models to Reason in a Continuous Latent Space},
  author    = {Hao, Shibo and Sukhbaatar, Sainbayar and Su, DiJia and Li, Xian and Hu, Zhiting and Weston, Jason and Tian, Yuandong},
  booktitle = {Conference on Language Modeling},
  year      = {2025}
}

@article{touvron2023llama,
  title={{Llama 2}: Open Foundation and Fine-Tuned Chat Models},
  author={Touvron, Hugo and Martin, Louis and Stone, Kevin and Albert, Peter and Almahairi, Amjad and Babaei, Yasmine and Bashlykov, Nikolay and Batra, Soumya and Bhargava, Prajjwal and Bhosale, Shruti and others},
  journal={arXiv preprint arXiv:2307.09288},
  year={2023}
}

@article{bai2023qwen,
  title={{Qwen} Technical Report},
  author={Bai, Jinze and Bai, Shuai and Chu, Yunfei and Cui, Zeyu and Dang, Kai and Deng, Xiaodong and Fan, Yang and Ge, Wenbin and Han, Yu and Huang, Fei and others},
  journal={arXiv preprint arXiv:2309.16609},
  year={2023}
}

@article{liu2024deepseek,
  title={{DeepSeek-V3} Technical Report},
  author={Liu, Aixin and Feng, Bei and Xue, Bing and Wang, Bingxuan and Wu, Bochao and Lu, Chengda and Zhao, Chenggang and Deng, Chengqi and Zhang, Chenyu and Ruan, Chong and others},
  journal={arXiv preprint arXiv:2412.19437},
  year={2024}
}

@article{blakeman2025nvidia,
  title={{NVIDIA} {N}emotron 3: Efficient and Open Intelligence},
  author={Blakeman, Aaron and Grattafiori, Aaron and Basant, Aarti and Gupta, Abhibha and Khattar, Abhinav and Renduchintala, Adi and Vavre, Aditya and Shukla, Akanksha and Bercovich, Akhiad and Ficek, Aleksander and others},
  journal={arXiv preprint arXiv:2512.20856},
  year={2025}
}

@article{zeng2026glm,
  title={{GLM}-5: From Vibe Coding to Agentic Engineering},
  author={Zeng, Aohan and Lv, Xin and Hou, Zhenyu and Du, Zhengxiao and Zheng, Qinkai and Chen, Bin and Yin, Da and Ge, Chendi and Huang, Chenghua and Xie, Chengxing and others},
  journal={arXiv preprint arXiv:2602.15763},
  year={2026}
}

@inproceedings{gloeckle2024better,
  title={Better \& faster large language models via multi-token prediction},
  author={Gloeckle, Fabian and Idrissi, Badr Youbi and Rozi{\`e}re, Baptiste and Lopez-Paz, David and Synnaeve, Gabriel},
  booktitle={International Conference on Machine Learning},
  year={2024}
}

@inproceedings{cai2024medusa,
  title={{MEDUSA}: Simple {LLM} Inference Acceleration Framework with Multiple Decoding Heads},
  author={Cai, Tianle and Li, Yuhong and Geng, Zhengyang and Peng, Hongwu and Lee, Jason D and Chen, Deming and Dao, Tri},
  booktitle={International Conference on Machine Learning},
  year={2024}
}

@inproceedings{li2024eagle,
  title={{EAGLE}: Speculative Sampling Requires Rethinking Feature Uncertainty},
  author={Li, Yuhui and Wei, Fangyun and Zhang, Chao and Zhang, Hongyang},
  booktitle={International Conference on Machine Learning},
  year={2024}
}

@misc{glm52,
  title        = {{GLM}-5.2: Built for Long-Horizon Tasks},
  author       = {{GLM Team}},
  year         = {2026},
  howpublished = {\url{https://z.ai/blog/glm-5.2}},
  note         = {Z.ai blog}
}

@inproceedings{saunshi2025reasoning,
  title={Reasoning with latent thoughts: On the power of looped transformers},
  author={Saunshi, Nikunj and Dikkala, Nishanth and Li, Zhiyuan and Kumar, Sanjiv and J Reddi, Sashank},
  booktitle={International Conference on Learning Representations},
  year={2025}
}

@inproceedings{dehghaniuniversal,
  title={Universal Transformers},
  author={Dehghani, Mostafa and Gouws, Stephan and Vinyals, Oriol and Uszkoreit, Jakob and Kaiser, Lukasz},
  booktitle={International Conference on Learning Representations},
  year={2019}
}

@inproceedings{lanalbert,
  title={{ALBERT}: A Lite {BERT} for Self-supervised Learning of Language Representations},
  author={Lan, Zhenzhong and Chen, Mingda and Goodman, Sebastian and Gimpel, Kevin and Sharma, Piyush and Soricut, Radu},
  booktitle={International Conference on Learning Representations},
  year={2020}
}

@article{zhu2025scaling,
  title={Scaling latent reasoning via looped language models},
  author={Zhu, Rui-Jie and Wang, Zixuan and Hua, Kai and Zhang, Tianyu and Li, Ziniu and Que, Haoran and Wei, Boyi and Wen, Zixin and Yin, Fan and Xing, He and others},
  journal={arXiv preprint arXiv:2510.25741},
  year={2025}
}

@article{qiu2025gated,
  title={Gated Attention for Large Language Models: Non-linearity, Sparsity, and Attention-Sink-Free},
  author={Qiu, Zihan and Wang, Zekun and Zheng, Bo and Huang, Zeyu and Wen, Kaiyue and Yang, Songlin and Men, Rui and Yu, Le and Huang, Fei and Huang, Suozhi and others},
  journal={Advances in Neural Information Processing Systems},
  volume={38},
  pages={100092--100118},
  year={2025}
}

@article{su2024roformer,
  title={{RoFormer}: Enhanced Transformer with Rotary Position Embedding},
  author={Su, Jianlin and Ahmed, Murtadha and Lu, Yu and Pan, Shengfeng and Bo, Wen and Liu, Yunfeng},
  journal={Neurocomputing},
  volume={568},
  pages={127063},
  year={2024},
  publisher={Elsevier}
}

@inproceedings{dai2024deepseekmoe,
  title={{DeepSeekMoE}: Towards Ultimate Expert Specialization in Mixture-of-Experts Language Models},
  author={Dai, Damai and Deng, Chengqi and Zhao, Chenggang and Xu, RX and Gao, Huazuo and Chen, Deli and Li, Jiashi and Zeng, Wangding and Yu, Xingkai and Wu, Yu and others},
  booktitle={Proceedings of the 62nd Annual Meeting of the Association for Computational Linguistics (Volume 1: Long Papers)},
  pages={1280--1297},
  year={2024}
}

@article{shazeer2020glu,
  title={{GLU} Variants Improve Transformer},
  author={Shazeer, Noam},
  journal={arXiv preprint arXiv:2002.05202},
  year={2020}
}

@misc{jordan2024muon,
  title  = {Muon: An optimizer for hidden layers in neural networks},
  author = {Jordan, Keller and Jin, Yuchen and Boza, Vlado and You, Jiacheng and Cesista, Franz and Newhouse, Laker and Bernstein, Jeremy},
  year   = {2024},
  url    = {https://kellerjordan.github.io/posts/muon/}
}

@article{wang2024mmlu,
  title={{MMLU-Pro}: A More Robust and Challenging Multi-Task Language Understanding Benchmark},
  author={Wang, Yubo and Ma, Xueguang and Zhang, Ge and Ni, Yuansheng and Chandra, Abhranil and Guo, Shiguang and Ren, Weiming and Arulraj, Aaran and He, Xuan and Jiang, Ziyan and others},
  journal={Advances in Neural Information Processing Systems},
  volume={37},
  pages={95266--95290},
  year={2024}
}

@inproceedings{hendrycksmeasuring,
  title={Measuring Massive Multitask Language Understanding},
  author={Hendrycks, Dan and Burns, Collin and Basart, Steven and Zou, Andy and Mazeika, Mantas and Song, Dawn and Steinhardt, Jacob},
  booktitle={International Conference on Learning Representations},
    year={2021}
}

@article{sakaguchi2021winogrande,
  title={{W}ino{G}rande: An Adversarial {W}inograd Schema Challenge at Scale},
  author={Sakaguchi, Keisuke and Bras, Ronan Le and Bhagavatula, Chandra and Choi, Yejin},
  journal={Communications of the ACM},
  volume={64},
  number={9},
  pages={99--106},
  year={2021},
  publisher={ACM New York, NY, USA}
}

@inproceedings{kazemi2025big,
  title={{BIG}-{B}ench Extra Hard},
  author={Kazemi, Mehran and Fatemi, Bahare and Bansal, Hritik and Palowitch, John and Anastasiou, Chrysovalantis and Mehta, Sanket Vaibhav and Jain, Lalit K and Aglietti, Virginia and Jindal, Disha and Chen, Yuanzhu Peter and others},
  booktitle={Proceedings of the 63rd Annual Meeting of the Association for Computational Linguistics (Volume 1: Long Papers)},
  pages={26473--26501},
  year={2025}
}

@inproceedings{suzgun2023challenging,
  title={Challenging {BIG}-{B}ench Tasks and Whether Chain-of-Thought Can Solve Them},
  author={Suzgun, Mirac and Scales, Nathan and Sch{\"a}rli, Nathanael and Gehrmann, Sebastian and Tay, Yi and Chung, Hyung Won and Chowdhery, Aakanksha and Le, Quoc and Chi, Ed and Zhou, Denny and others},
  booktitle={Findings of the Association for Computational Linguistics: ACL 2023},
  pages={13003--13051},
  year={2023}
}

@inproceedings{reingpqa,
  title={{GPQA}: A Graduate-Level {G}oogle-{P}roof {Q\&A} Benchmark},
  author={Rein, David and Hou, Betty Li and Stickland, Asa Cooper and Petty, Jackson and Pang, Richard Yuanzhe and Dirani, Julien and Michael, Julian and Bowman, Samuel R},
  booktitle={First Conference on Language Modeling},
    year={2024}
}

@inproceedings{hendrycks2measuring,
  title={Measuring Mathematical Problem Solving With the {MATH} Dataset},
  author={Hendrycks, Dan and Burns, Collin and Kadavath, Saurav and Arora, Akul and Basart, Steven and Tang, Eric and Song, Dawn and Steinhardt, Jacob},
  booktitle={Thirty-fifth Conference on Neural Information Processing Systems Datasets and Benchmarks Track (Round 2)},
    year={2021}
}

@inproceedings{yang2018hotpotqa,
  title={{HotpotQA}: A Dataset for Diverse, Explainable Multi-Hop Question Answering},
  author={Yang, Zhilin and Qi, Peng and Zhang, Saizheng and Bengio, Yoshua and Cohen, William and Salakhutdinov, Ruslan and Manning, Christopher D},
  booktitle={Proceedings of the 2018 conference on empirical methods in natural language processing},
  pages={2369--2380},
  year={2018}
}

@inproceedings{mallen2023not,
  title={When not to trust language models: Investigating effectiveness of parametric and non-parametric memories},
  author={Mallen, Alex and Asai, Akari and Zhong, Victor and Das, Rajarshi and Khashabi, Daniel and Hajishirzi, Hannaneh},
  booktitle={Proceedings of the 61st annual meeting of the association for computational linguistics (volume 1: Long papers)},
  pages={9802--9822},
  year={2023}
}

@inproceedings{bang2025hallulens,
  title={{H}allu{L}ens: {LLM} Hallucination Benchmark},
  author={Bang, Yejin and Ji, Ziwei and Schelten, Alan and Hartshorn, Anthony and Fowler, Tara and Zhang, Cheng and Cancedda, Nicola and Fung, Pascale},
  booktitle={Proceedings of the 63rd Annual Meeting of the Association for Computational Linguistics (Volume 1: Long Papers)},
  pages={24128--24156},
  year={2025}
}

@article{wei2024measuring,
  title={Measuring short-form factuality in large language models},
  author={Wei, Jason and Karina, Nguyen and Chung, Hyung Won and Jiao, Yunxin Joy and Papay, Spencer and Glaese, Amelia and Schulman, John and Fedus, William},
  journal={arXiv preprint arXiv:2411.04368},
  year={2024}
}

@article{liu2023your,
  title={Is your code generated by {ChatGPT} really correct? {R}igorous evaluation of large language models for code generation},
  author={Liu, Jiawei and Xia, Chunqiu Steven and Wang, Yuyao and Zhang, Lingming},
  journal={Advances in neural information processing systems},
  volume={36},
  pages={21558--21572},
  year={2023}
}

@inproceedings{jain2025livecodebench,
  title={{LiveCodeBench}: Holistic and Contamination Free Evaluation of Large Language Models for Code},
  author={Jain, Naman and Han, King and Gu, Alex and Li, Wen-Ding and Yan, Fanjia and Zhang, Tianjun and Wang, Sida and Solar-Lezama, Armando and Sen, Koushik and Stoica, Ion},
  booktitle={International Conference on Learning Representations},
  year={2025}
}

@misc{chen2026dflashblockdiffusionflash,
      title={{DFlash}: Block Diffusion for Flash Speculative Decoding}, 
      author={Jian Chen and Yesheng Liang and Zhijian Liu},
      year={2026},
      eprint={2602.06036},
      archivePrefix={arXiv},
      primaryClass={cs.CL},
      url={https://arxiv.org/abs/2602.06036}, 
}

@misc{cheng2026dsparkconfidencescheduledspeculativedecoding,
      title={{DSpark}: Confidence-Scheduled Speculative Decoding with Semi-Autoregressive Generation}, 
      author={Xin Cheng and Xingkai Yu and Chenze Shao and Jiashi Li and Yunfan Xiong and Yi Qian and Jiaqi Zhu and Shirong Ma and Xiaokang Zhang and Jiasheng Ye and Qinyu Chen and Chengqi Deng and Jiping Yu and Damai Dai and Zhengyan Zhang and Yixuan Wei and Yixuan Tan and Wenkai Yang and Runxin Xu and Yu Wu and Zhean Xu and Xuanyu Wang and Muyang Chen and Rui Tian and Xiao Bi and Zhewen Hao and Shaoyuan Chen and Huanqi Cao and Wentao Zhang and Anyi Xu and Huishuai Zhang and Dongyan Zhao and Wenfeng Liang},
      year={2026},
      eprint={2607.05147},
      archivePrefix={arXiv},
      primaryClass={cs.AI},
      url={https://arxiv.org/abs/2607.05147}, 
}

@misc{li2026specforge,
    title={{SpecForge}: A Flexible and Efficient Open-Source Training Framework for Speculative Decoding},
    author={Shenggui Li and Chao Wang and Yikai Zhu and Yubo Wang and Fan Yin and Shuai Shi and Yefei Chen and Xiaomin Dong and Qiaoling Chen and Jin Pan and Ji Li and Laixin Xie and Yineng Zhang and Lei Yu and Yonggang Wen and Ivor Tsang and Tianwei Zhang},
    year={2026},
    eprint={2603.18567},
    archivePrefix={arXiv},
    primaryClass={cs.LG},
    url={https://arxiv.org/abs/2603.18567},
}

@misc{kimiteam2026k3,
      title={{Kimi K3}: Open Frontier Intelligence}, 
      author={{Kimi Team}},
      year={2026},
      eprint={2607.24653},
      archivePrefix={arXiv},
      primaryClass={cs.CL},
      url={https://arxiv.org/abs/2607.24653}, 
}

@misc{huang2026dominodecouplingcausalmodeling,
      title={Domino: Decoupling Causal Modeling from Autoregressive Drafting in Speculative Decoding}, 
      author={Jianuo Huang and Yaojie Zhang and Qituan Zhang and Hao Lin and Hanlin Xu and Linfeng Zhang},
      year={2026},
      eprint={2605.29707},
      archivePrefix={arXiv},
      primaryClass={cs.CL},
      url={https://arxiv.org/abs/2605.29707}, 
}

@inproceedings{leviathan2023fast,
  title     = {Fast Inference from Transformers via Speculative Decoding},
  author    = {Leviathan, Yaniv and Kalman, Matan and Matias, Yossi},
  booktitle = {Proceedings of the 40th International Conference on Machine Learning},
  series    = {Proceedings of Machine Learning Research},
  volume    = {202},
  pages     = {19274--19286},
  year      = {2023}
}

@article{wu2026dpace,
  title   = {{D-PACE}: Dynamic Position-Aware Cross-Entropy for Parallel Speculative Drafting},
  author  = {Wu, Tianyu and Yao, Yu and Qi, Zhenting and Zheng, Han and Wang, Zhuohan and Ma, Haoran and Liao, Lawrence and Lakkaraju, Himabindu and Li, Ju and Du, Yilun},
  journal = {arXiv preprint arXiv:2605.18810},
  year    = {2026}
}

@misc{zhou2026overcomingjointintractabilitylossless,
      title={Overcoming Joint Intractability with Lossless Hierarchical Speculative Decoding}, 
      author={Yuxuan Zhou and Fei Huang and Heng Li and Fengyi Wu and Tianyu Wang and Jianwei Zhang and Junyang Lin and Zhi-Qi Cheng},
      year={2026},
      eprint={2601.05724},
      archivePrefix={arXiv},
      primaryClass={cs.AI},
      url={https://arxiv.org/abs/2601.05724}, 
}

@inproceedings{huang2016deep,
  title     = {Deep Networks with Stochastic Depth},
  author    = {Huang, Gao and Sun, Yu and Liu, Zhuang and Sedra, Daniel and Weinberger, Kilian Q.},
  booktitle = {European Conference on Computer Vision (ECCV)},
  pages     = {646--661},
  year      = {2016}
}

@misc{liu2026angelspecrealworldhighperformance,
      title={{AngelSpec}: Towards Real-World High Performance Inference with Speculative Decoding}, 
      author={Hong Liu and Rui Cen and Junhan Shi and Guangshuo Qin and Jiebin Zhang and Tianyu Liu and Runzhi Fan and Guoliang Zhao and Ruobing Xie and Kai Zhang and Song Liu and Guanghua Yu and Jianchen Zhu},
      year={2026},
      eprint={2607.25852},
      archivePrefix={arXiv},
      primaryClass={cs.CL},
      url={https://arxiv.org/abs/2607.25852}, 
}

@techreport{chatterji2025people,
  title       = {How People Use {ChatGPT}},
  author      = {Chatterji, Aaron and Cunningham, Thomas and Deming, David J. and Hitzig, Zo{\"e} and Ong, Christopher and Shan, Carl Yan and Wadman, Kevin},
  institution = {National Bureau of Economic Research},
  type        = {Working Paper},
  number      = {34255},
  year        = {2025}
}

@article{xiaomi2025mimo,
   title={{MiMo}: Unlocking the Reasoning Potential of Language Model -- From Pretraining to Posttraining},
   author={{LLM-Core Xiaomi}}, journal={arXiv preprint arXiv:2505.07608}, year={2025}}

@inproceedings{geiping2025scaling,
  title     = {Scaling up Test-Time Compute with Latent Reasoning: A Recurrent Depth Approach},
  author    = {Geiping, Jonas and McLeish, Sean and Jain, Neel and Kirchenbauer, John and Singh, Siddharth and Bartoldson, Brian R. and Kailkhura, Bhavya and Bhatele, Abhinav and Goldstein, Tom},
  booktitle = {Advances in Neural Information Processing Systems},
  year      = {2025}
}

@inproceedings{bae2025mixture,
  title     = {Mixture-of-Recursions: Learning Dynamic Recursive Depths for Adaptive Token-Level Computation},
  author    = {Bae, Sangmin and Kim, Yujin and Bayat, Reza and Kim, Sungnyun and Ha, Jiyoun and Schuster, Tal and Fisch, Adam and Harutyunyan, Hrayr and Ji, Ziwei and Courville, Aaron and Yun, Se-Young},
  booktitle = {Advances in Neural Information Processing Systems},
  year      = {2025}
}

@inproceedings{mehra2025multitoken,
  title     = {On Multi-Token Prediction for Efficient {LLM} Inference},
  author    = {Mehra, Somesh and Garcia, Javier Alonso and Mauch, Lukas},
  booktitle = {ICLR 2025 Workshop on Sparsity in LLMs (SLLM)},
  year      = {2025},
  note      = {arXiv:2502.09419}
}

@inproceedings{li2025eagle3,
  title   = {{EAGLE-3}: Scaling up Inference Acceleration of Large Language Models via Training-Time Test},
  author  = {Li, Yuhui and Wei, Fangyun and Zhang, Chao and Zhang, Hongyang},
  booktitle = {Advances in Neural Information Processing Systems},
  year    = {2025}
}

@article{hong2025training,
  title   = {Training Domain Draft Models for Speculative Decoding: Best Practices and Insights},
  author  = {Hong, Fenglu and Raju, Ravi and Li, Jonathan Lingjie and Li, Bo and Thakker, Urmish and Ravichandran, Avinash and Jain, Swayambhoo and Hu, Changran},
  journal = {arXiv preprint arXiv:2503.07807},
  year    = {2025}
}

@inproceedings{liu2024online,
  title     = {Online Speculative Decoding},
  author    = {Liu, Xiaoxuan and Hu, Lanxiang and Bailis, Peter and Cheung, Alvin and Deng, Zhijie and Stoica, Ion and Zhang, Hao},
  booktitle = {Proceedings of the 41st International Conference on Machine Learning},
  series    = {Proceedings of Machine Learning Research},
  volume    = {235},
  year      = {2024}
}

@inproceedings{zhou2024distillspec,
  title     = {{DistillSpec}: Improving Speculative Decoding via Knowledge Distillation},
  author    = {Zhou, Yongchao and Lyu, Kaifeng and Rawat, Ankit Singh and Menon, Aditya Krishna and Rostamizadeh, Afshin and Kumar, Sanjiv and Kagy, Jean-Fran{\c{c}}ois and Agarwal, Rishabh},
  booktitle = {International Conference on Learning Representations},
  year      = {2024}
}

@inproceedings{bie2026osprey,
  title     = {Osprey: Target-agnostic Pre-training Makes Stronger Drafters in Speculative Decoding},
  author    = {Bie, Fengxiang and Jian, Yuqing and Yu, Yifan and Zhou, Zhongzhu and Shao, Zelei and Athiwaratkun, Ben and Song, Shuaiwen Leon and Xu, Chenfeng and Wu, Xiaoxia and Zhang, Tianyi},
  booktitle = {Proceedings of the 2026 Conference on Empirical Methods in Natural Language Processing},
  year      = {2026},
  note      = {arXiv:2609.09338}
}

@article{cobbe2021gsm8k,
  title   = {Training Verifiers to Solve Math Word Problems},
  author  = {Cobbe, Karl and Kosaraju, Vineet and Bavarian, Mohammad and Chen, Mark and Jun, Heewoo and Kaiser, Lukasz and Plappert, Matthias and Tworek, Jerry and Hilton, Jacob and Nakano, Reiichiro and Hesse, Christopher and Schulman, John},
  journal = {arXiv preprint arXiv:2110.14168},
  year    = {2021}
}

@inproceedings{bisk2020piqa,
  title     = {{PIQA}: Reasoning about Physical Commonsense in Natural Language},
  author    = {Bisk, Yonatan and Zellers, Rowan and Le Bras, Ronan and Gao, Jianfeng and Choi, Yejin},
  booktitle = {Proceedings of the AAAI Conference on Artificial Intelligence},
  year      = {2020}
}

@inproceedings{wei2022chain,
  title     = {Chain-of-Thought Prompting Elicits Reasoning in Large Language Models},
  author    = {Wei, Jason and Wang, Xuezhi and Schuurmans, Dale and Bosma, Maarten and Ichter, Brian and Xia, Fei and Chi, Ed and Le, Quoc V and Zhou, Denny},
  booktitle = {Advances in Neural Information Processing Systems},
  year      = {2022}
}

@inproceedings{brown2020language,
  title     = {Language Models are Few-Shot Learners},
  author    = {Brown, Tom B. and Mann, Benjamin and Ryder, Nick and Subbiah, Melanie and Kaplan, Jared and Dhariwal, Prafulla and Neelakantan, Arvind and Shyam, Pranav and Sastry, Girish and Askell, Amanda and others},
  booktitle = {Advances in Neural Information Processing Systems},
  year      = {2020}
}

@misc{gao2024evalharness,
  title        = {A framework for few-shot language model evaluation},
  author       = {Gao, Leo and Tow, Jonathan and Abbasi, Baber and Biderman, Stella and Black, Sid and DiPofi, Anthony and Foster, Charles and Golding, Laurence and Hsu, Jeffrey and Le Noac'h, Alain and Li, Haonan and McDonell, Kyle and Muennighoff, Niklas and Ociepa, Chris and Phang, Jason and Reynolds, Laria and Schoelkopf, Hailey and Skowron, Aviya and Sutawika, Lintang and Tang, Eric and Thite, Anish and Wang, Ben and Wang, Kevin and Zou, Andy},
  year         = {2024},
  month        = {07},
  howpublished = {Zenodo, version v0.4.3},
  doi          = {10.5281/zenodo.12608602},
  url          = {https://zenodo.org/records/12608602}
}

@article{chen2021codex,
  title   = {Evaluating Large Language Models Trained on Code},
  author  = {Chen, Mark and Tworek, Jerry and Jun, Heewoo and Yuan, Qiming and Pinto, Henrique Ponde de Oliveira and Kaplan, Jared and Edwards, Harri and Burda, Yuri and Joseph, Nicholas and Brockman, Greg and others},
  journal = {arXiv preprint arXiv:2107.03374},
  year    = {2021}
}

@inproceedings{samarin2026lklosses,
  title     = {LK Losses: Direct Acceptance Rate Optimization for Speculative Decoding},
  author    = {Samarin, Alexander and Krutikov, Sergei and Shevtsov, Anton and Skvortsov, Sergei and Fisin, Filipp and Golubev, Alexander},
  booktitle = {Proceedings of the 43rd International Conference on Machine Learning},
  year      = {2026}
}

@misc{li2026breakingentropyboundsaccelerating,
      title={Breaking Entropy Bounds: Accelerating RL Training via MTP with Rejection Sampling}, 
      author={Yucheng Li and Huiqiang Jiang and Yang Xu and Jianxin Yang and Yi Zhang and Yizhong Cao and Yuhao Shen and Fan Zhou and Rui Men and Jianwei Zhang and An Yang and Bowen Yu and Bo Zheng and Fei Huang and Junyang Lin and Dayiheng Liu and Jingren Zhou},
      year={2026},
      eprint={2606.12370},
      archivePrefix={arXiv},
      primaryClass={cs.LG},
      url={https://arxiv.org/abs/2606.12370}, 
}
\bibliographystyle{colm2024_conference}

\clearpage
\appendix

\section{Experiment Setup}
\label{sec:appendix-setup}

\subsection{Pretraining}

\noindent \textbf{Model Architecture}: The backbone is a 10B-A1B Latent Mixture-of-Experts transformer \citep{blakeman2025nvidia}, that is 10B parameters in total with 1B activated per token, stacked as 24 layers of hidden size 1{,}024 with 16 attention heads. Attention is multi-head latent \citep{liu2024deepseek}, with 512-dimensional query and key--value latents, and we add gated attention \citep{qiu2025gated}, per-head query--key normalization and rotary positional embeddings \citep{su2024roformer}. Every layer is a mixture of experts in a latent space: the expert block projects the hidden state down to 512 dimensions, runs 256 fine-grained routed experts \citep{dai2024deepseekmoe} of intermediate size 1{,}024 there, and projects back, with the router reading the hidden state rather than the latent. A sigmoid gate sends each token to eight routed experts, balanced by a sequence-level auxiliary loss, plus one shared expert. The SwiGLU \citep{shazeer2020glu} feed-forward width is 5{,}120 and the vocabulary holds 201{,}216 tokens.

\noindent \textbf{Training Settings}: We train the model on 600B tokens with a context sequence length of 4,096 and a global batch size of 1,920. Training uses bfloat16 precision and the Muon optimizer \citep{jordan2024muon}, with AdamW retained for the embeddings, the biases, the output head and the expert router, a learning rate of $6\times10^{-4}$, a warmup of 1,000 steps, weight decay of $0.1$, and gradient clipping at $1.0$. Following \citet{liu2024deepseek}, we set the MTP loss scaling factor to $0.3$. Pretraining runs on 24 nodes of eight H800 GPUs each, with eight-way expert parallelism, stage-one optimizer sharding, no tensor or pipeline parallelism, and full activation recomputation.

\noindent \textbf{Computational Cost}:
Each pretraining run is conducted on 24 nodes with eight NVIDIA H800 80GB GPUs per node (192 GPUs in total), taking approximately three days to complete. Supervised finetuning uses the same hardware configuration and takes approximately five hours per run.
For the post-hoc speculative decoding baselines, we train the draft models on three nodes (24 H800 80GB GPUs), using a global batch size of 48. EAGLE-3 (cold-start and warm-start) and DFlash use a per-GPU batch size of 2 without gradient accumulation, whereas DSpark uses a per-GPU batch size of 1 with two gradient accumulation steps due to GPU memory constraints. EAGLE-3 requires approximately 22 hours per run in measured wall-clock time. Based on the recorded compute time and an estimated 1.27$\times$ overhead for startup, checkpointing, and cluster variability, DSpark requires approximately 55 hours, while DFlash requires approximately 19 hours.

\noindent \textbf{Methods Compared}: All six configurations are trained under the recipe above, so the comparison isolates the MTP design. ECHO reuses the backbone's last layer as the body of every MTP depth; DeepSeek's MTP baselines follow \citet{liu2024deepseek} and keep an independent module per depth, sharing only the embedding and the output head. ECHO-1 and MTP-1 train one depth, ECHO-3 and MTP-3 train three at every step, ECHO-1.5$^{(3)}$ samples depths one to three with $\mathbb{E}[D]=1.5$, and ECHO-2$^{(7)}$ extends dynamic depth sampling to depths one through seven with $\mathbb{E}[D]=2$. The two cost-matched comparisons are ECHO-1 against MTP-1 at equal depth-one cost, and ECHO-1.5$^{(3)}$ against MTP-3 at half the MTP cost.

\noindent \textbf{Dynamic Depth Sampling}: The sampled depth is $D\sim\pi$, where $\pi$ places probabilities 0.6, 0.3 and 0.1 on depths 1, 2 and 3, so $\mathbb{E}[D]=1.5$. Each optimizer step draws one $D$ and trains depths 1 through $D$, and every micro-batch in the step shares the same $D$, which keeps per-step compute constant. Averaging the loss over the $D$ trained depths gives depth $d$ the weight $\mathbb{E}[\mathbf{1}\{D\ge d\}/D]$, and these weights sum to one: ECHO reweights depths by how often it samples them, not by attaching a per-depth loss coefficient. $D$ is a deterministic function of the seed and the step index, so every rank derives it without communication and a resumed run continues the same sequence. 

\noindent \textbf{Supervised Finetuning}: All configurations adopt an identical recipe: 5,000 steps over one pass of an in-house chat SFT corpus at a context length of 32,768, with eight-way context parallelism. Each keeps the MTP depths it was pretrained with, so the finetuned comparison isolates the same design difference as the pretrained one.

\subsection{Benchmark Evaluation}

\noindent \textbf{Evaluation Setup}: We evaluate the pretrained checkpoints on fourteen few-shot benchmarks, grouped by the form of the task following \citet{liu2024deepseek}: mathematics with MATH \citep{hendrycks2measuring} and GSM8K \citep{cobbe2021gsm8k}; reasoning with BBH \citep{suzgun2023challenging}, BBEH \citep{kazemi2025big}, WinoGrande \citep{sakaguchi2021winogrande} and PIQA \citep{bisk2020piqa}; multi-subject multiple choice with MMLU \citep{hendrycksmeasuring} and MMLU-Pro \citep{wang2024mmlu}; code with HumanEval+ and MBPP+ \citep{liu2023your} and LiveCodeBench \citep{jain2025livecodebench}; closed-book question answering with PopQA \citep{mallen2023not}, PreciseWikiQA \citep{bang2025hallulens} and HotpotQA \citep{yang2018hotpotqa}. We report unweighted averages within each group and across all fourteen benchmarks. Benchmarks such as GPQA \citep{reingpqa}, SimpleQA \citep{wei2024measuring} and HMMT are left out for lack of resolution rather than for any defect in their scoring: every configuration scores near their floor, so the row moves more than it resolves.

\noindent \textbf{Prompts and Decoding}: All benchmarks are evaluated in the few-shot in-context format \citep{brown2020language}. Chain-of-thought prompting \citep{wei2022chain} is used on MATH, GSM8K and BBH, with eight, eight and four demonstrations respectively. The remaining benchmarks are prompted to answer directly, with five demonstrations for the multiple-choice benchmarks, ten for closed-book question answering and two for BBEH. BBH is released under both a chain-of-thought and an answer-directly prompt over the same items for examining the chain-of-thought effect \citep{suzgun2023challenging}; we report its chain-of-thought form in the main results. Decoding is greedy throughout. Generation is capped per benchmark: two tokens for MMLU and eleven for MMLU-Pro, which are answered with a single option label; $65$ for the closed-book sets and $101$ for PopQA; $401$ for MATH, $512$ for GSM8K, BBEH and HumanEval+, $601$ for BBH, $768$ for LiveCodeBench and $8192$ for MBPP+. WinoGrande and PIQA involve no generation. Stop sequences are applied when the answer is extracted rather than during decoding, so the cap is what bounds each output.

\noindent \textbf{Scoring and Averaging}: Each benchmark's score is an unweighted macro-average over its sub-tasks rather than a per-item micro-average. Four scoring rules are used, one per group of benchmarks. GSM8K, MATH, BBH, MMLU and MMLU-Pro are scored by rule from the extracted answer. HumanEval+, MBPP+ and LiveCodeBench are scored as pass@1 by execution \citep{chen2021codex}. WinoGrande and PIQA are scored by length-normalized log-likelihood over the candidate continuations rather than by generation, as in \citet{liu2024deepseek}, with the normalization of \citet{gao2024evalharness}. PreciseWikiQA, HotpotQA, PopQA and BBEH are graded against the reference answer by an LLM judge. In \Cref{tab:benchmark_results} the three factual-knowledge sets are graded by gpt-oss-120b and BBEH by Claude Opus 4.8; in \Cref{tab:sft} all four are graded by Claude Opus 4.8. Of these, the three factual-knowledge sets report the F1 of overall accuracy and accuracy among attempted items, with abstentions in the denominator of the former only. Category averages and the overall figure are unweighted means over benchmarks.

\subsection{Speculative Decoding under Serving}

\noindent \textbf{Serving Stack}: All serving measurements use SGLang on one node of eight H800 GPUs, with tensor parallelism~8 and FlashAttention-3.

\noindent \textbf{Drafters}: All drafters are budget-matched at seven proposed positions plus one target bonus token, i.e.\ eight verification slots per request per step. ECHO reuses the target's final layer and full output head. EAGLE-3~\citep{li2025eagle3} uses one draft layer with a reduced 32{,}000-entry head scattered back to target width, in a cold-start and a warm-start variant; following \citet{kimiteam2026k3}, the warm start initializes the draft from DeepSeek's pretrained MTP weights, and is always normalized by the baseline of its own run. DSpark~\citep{cheng2026dsparkconfidencescheduledspeculativedecoding} and DFlash~\citep{chen2026dflashblockdiffusionflash} are block-parallel with five draft layers, the former adding a low-rank Markov transition bias.

\noindent \textbf{Serving Metric}: We report decode-only throughput, decoded tokens over summed per-request decode latency, which excludes prefill; metrics that include prefill confound the comparison across tasks with differing prompt lengths. Speedup is recomputed against the baseline of the same run, and accepted length is decoded tokens per verification step.

\noindent \textbf{Tasks}: Serving results use ten tasks in two groups, split by predictive entropy as in \citet{liu2026angelspecrealworldhighperformance}. Predictable, following the commonly adopted SpecForge benchmarks~\citep{li2026specforge}: HumanEval, MBPP, MATH-500, GSM8K, LiveCodeBench. Open-ended: SWE-bench, Arena-Hard, MT-Bench, summarization, Alpaca. Question answering, retrieval-augmented generation, translation and AIME-25 are excluded, all having median outputs of twenty to thirty tokens, too few verification steps for a stable decode-only ratio. The entropy axis is the target's own predictive entropy, measured separately.

\noindent \textbf{Verification Budget}: DSpark is evaluated with its confidence scheduler disabled, verifying the full draft block at every step, as in the diagnostic baseline of \citet{cheng2026dsparkconfidencescheduledspeculativedecoding}. The scheduler is a calibrated cost model rather than a threshold: it needs a latency curve re-profiled for each GPU and backend \citep{cheng2026dsparkconfidencescheduledspeculativedecoding,liu2026angelspecrealworldhighperformance}, and a confidence head calibrated on the serving distribution, since the survival probabilities it multiplies must be absolute rather than merely well-ranked. That calibration does not transfer, because the mechanism assumes acceptance varies sharply across domains in the first place. With the head enabled but neither artifact available, per-step cost rose 5--14\% and nothing was pruned.

\noindent \textbf{Depth Extrapolation}: The depth sweep of \Cref{fig:extra} uses a separate harness. We report the accepted length obtained at serving depth $K$, $\tau_K = 1 + \sum_{j\le K}\prod_{i\le j} p_i$, for conditional acceptance rate $p_i$ at depth $i$. The expression is exact rather than a fit: the draft loop reaches depth $i+1$ only after depth $i$ is accepted, so $\prod_{i\le j} p_i$ is the probability that the first $j$ drafted tokens are all accepted, and summing those probabilities gives their expected number. Runs use draft depth $K{=}7$, greedy on both sides, 256 new tokens, and pool GSM8K, HumanEval and MATH-500 weighted by rounds.

\section{Decayed Depth Supervision and Accepted Length}
\label{app:dds}

Let $p_k$, $q_k$, and $d_k$ be the target, draft, and data next-token distributions at depth $k$, each conditioned on the current prefix.
Let $\delta_k = 1 - a_k$, $c_k = K - k + 1$, and $A_k = \mathrm{KL}(d_k \,\Vert\, q_k)$, the drafter distributional loss at depth $k$.
We assume the drafter is trained on target-generated data, so $d_k = p_k$ and $A_k = \mathrm{KL}(p_k \,\Vert\, q_k)$.
All quantities are averaged over prefixes, which we assume follow the training data distribution at inference; each inequality below holds pointwise for every prefix and therefore after averaging, with Jensen's inequality moving the average inside square roots where needed.

\paragraph{Derivation of \cref{eq:chain}.}
The full chain is
\begin{equation}
  \label{eq:app-chain}
  \mathbb{E}[\tau]
  \ \overset{(i)}{\ge}\ 1 + K - \sum_{k=1}^{K} c_k\, \mathrm{TV}(p_k, q_k)
  \ \overset{(ii)}{\ge}\ 1 + K - \tfrac{1}{\sqrt{2}}\, Q
  \ \overset{(iii)}{\ge}\ 1 + K - \tfrac{1}{\sqrt{2}}\, \mathcal{L}_{\text{MTP}}(f)
  \ \overset{(iv)}{\ge}\ 1 + K - \tfrac{1}{\sqrt{2}}\, \mathcal{L}_{\text{MTP}}(\mathbf{1}),
\end{equation}
with $Q = \sum_k c_k \sqrt{A_k}$.

\emph{(i) Union bound.}
By \cref{eq:depth-match}, $1 + K - \mathbb{E}[\tau] = \sum_{d=1}^{K} \bigl( 1 - \prod_{i \le d} (1 - \delta_i) \bigr)$, where each term is the probability that draft $d$ is not accepted.
The probability of at least one rejection among depths $1, \dots, d$ is at most $\sum_{i \le d} \delta_i$, a bound tight up to second-order terms in $\delta$.
Summing over $d$, each $\delta_k$ is counted once for each of the $c_k$ values $d \ge k$, and $\delta_k = \mathrm{TV}(p_k, q_k)$ under lossless verification~\citep{leviathan2023fast}.

\emph{(ii) Pinsker's inequality.}
Since $d_k = p_k$, Pinsker's inequality gives $\mathrm{TV}(p_k, q_k) = \mathrm{TV}(d_k, q_k) \le \sqrt{A_k / 2}$.

\emph{(iii) Cauchy--Schwarz.}
Writing $c_k \sqrt{A_k} = (c_k / \sqrt{f_k}) \cdot \sqrt{f_k A_k}$,
\begin{equation}
  \label{eq:app-cs}
  Q \ \le\ \mathcal{L}_{\text{MTP}}(f) = \sqrt{\sum_{k=1}^{K} \frac{c_k^2}{f_k}}\ \sqrt{\sum_{k=1}^{K} f_k A_k},
\end{equation}
with equality if and only if $f_k \propto c_k / \sqrt{A_k}$.
Since $\mathcal{L}_{\text{MTP}}(\lambda f) = \mathcal{L}_{\text{MTP}}(f)$ for any $\lambda > 0$, schedules with different compute budgets can be compared directly.
For MTP-$K$, $\mathcal{L}_{\text{MTP}}(\mathbf{1}) = \sqrt{\sum_k c_k^2}\,\sqrt{\sum_k A_k}$, with $\sum_k c_k^2 = K(K+1)(2K+1)/6$.

\emph{(iv) Decayed schedules.}
This is \cref{prop:decay}(c).

\begin{proposition}
\label{prop:decay}
Suppose $K \ge 2$ and $0 < A_1 \le \dots \le A_K$, and let $f^\star_k \propto c_k / \sqrt{A_k}$. Then
(a) $f^\star$ is strictly decreasing in $k$;
(b) $\mathcal{L}_{\text{MTP}}(\mathbf{1}) > Q$;
(c) for $f^{(t)}_k = (f^\star_k)^t$, $\mathcal{L}_{\text{MTP}}(f^{(t)})$ is nonincreasing in $t \in [0, 1]$, from $\mathcal{L}_{\text{MTP}}(\mathbf{1})$ at $t = 0$ to $Q$ at $t = 1$.
\end{proposition}

\begin{proof}
(a) $c_k$ is strictly decreasing and $1/\sqrt{A_k}$ is nonincreasing, both positive.
(b) By (a), $\mathbf{1}$ is not proportional to $f^\star$, so equality in \cref{eq:app-cs} fails.
(c) With $u_k = \log f^\star_k$, the function $F(t) = 2 \log \mathcal{L}_{\text{MTP}}(f^{(t)}) = \log \sum_k c_k^2 e^{-t u_k} + \log \sum_k A_k e^{t u_k}$ is a sum of log-sum-exps of affine functions, hence convex, and is minimized at $t = 1$ by the equality condition of \cref{eq:app-cs}; a convex function is nonincreasing to the left of its minimizer.
\end{proof}

The assumption $A_1 \le \dots \le A_K$ reflects that acceptance decays with depth.
\Cref{prop:decay} concerns how tightly the training loss bounds lost length; that DDS attains comparable $\mathbb{E}[\tau]$ at lower cost is verified empirically.

\paragraph{Data not generated by the target.}
If $d_k \ne p_k$, step (ii) uses the triangle inequality $\mathrm{TV}(p_k, q_k) \le \mathrm{TV}(d_k, q_k) + \mathrm{TV}(d_k, p_k)$ and Pinsker on both terms, which adds $\tfrac{1}{\sqrt{2}} \sum_k c_k \sqrt{\mathrm{KL}(d_k \,\Vert\, p_k)}$ to every bound in \cref{eq:app-chain}.
This term depends on neither the drafter nor the schedule, so steps (iii), (iv), and \cref{prop:decay} are unchanged.

\section{Full Benchmark Results}

\Cref{tab:benchmark_results} gives the per-benchmark scores behind the category gains of \Cref{fig:buckets_ds_pretrain}, for all six pretrained configurations over the fourteen benchmarks of the reported scope.

\noindent \textbf{Provenance of the cells}: nine of the fourteen rows are scored by us under the protocol of \Cref{sec:appendix-setup}. The other five are reported by an internal evaluation platform and are reproduced here as it returned them: PreciseWikiQA, HotpotQA and PopQA, graded by gpt-oss-120b over its full sets of $5{,}000$, $7{,}405$ and $14{,}256$ items, together with MATH ($4{,}739$ items, scored by rule) and MBPP+ ($378$ items, scored by execution). That platform's answer-extraction and judging rules are not published with this paper, so these five rows cannot be reproduced exactly from outside it. Our own runs of them exist but measure a different quantity — a different few-shot prefix, an answer-extraction rule that does not match on MATH, a weaker test bar on MBPP+, and a $600$-item sample under a different judge on the three closed-book sets — so they are not offered as a second reading.

\begin{table}[!htbp]
  \centering
  \footnotesize
  \setlength{\tabcolsep}{4pt}
  \begin{tabular}{llrrrrrr}
    \toprule
    Category & Benchmark & ECHO-1 & ECHO-3 & ECHO-1.5$^{(3)}$ & MTP-1 & MTP-3 & ECHO-2$^{(7)}$ \\
    \midrule
    \multirow{2}{*}{Mathematics}
      & GSM8K           & 61.49 & 59.82 & 63.31 & 61.64 & 63.08 & 61.33 \\
      & MATH            & 36.04 & 36.01 & 34.64 & 34.31 & 33.78 & 33.56 \\
    \midrule
    \multirow{4}{*}{Reasoning}
      & BBH             & 46.44 & 50.02 & 48.55 & 46.91 & 40.18 & 48.45 \\
      & BBEH            & 6.63 & 6.41 & 6.09 & 5.33 & 6.96 & 5.43 \\
      & PIQA            & 77.57 & 78.26 & 78.38 & 77.68 & 77.97 & 78.09 \\
      & WinoGrande      & 63.85 & 64.01 & 64.40 & 62.27 & 63.06 & 62.59 \\
    \midrule
    \multirow{2}{*}{Multi-subject MC}
      & MMLU            & 55.93 & 57.10 & 56.93 & 56.39 & 58.59 & 56.95 \\
      & MMLU-Pro        & 26.02 & 26.37 & 27.58 & 26.60 & 29.37 & 27.90 \\
    \midrule
    \multirow{3}{*}{Code}
      & HumanEval+      & 26.22 & 29.88 & 29.27 & 23.17 & 25.00 & 26.83 \\
      & MBPP+           & 47.62 & 48.68 & 47.62 & 47.62 & 47.88 & 47.88 \\
      & LiveCodeBench   & 7.00 & 7.00 & 7.00 & 6.00 & 6.00 & 5.25 \\
    \midrule
    \multirow{3}{*}{Closed-book QA}
      & PreciseWikiQA   & 7.88 & 8.30 & 8.13 & 7.48 & 7.52 & 7.54 \\
      & HotpotQA        & 21.35 & 21.29 & 20.65 & 21.18 & 20.86 & 20.57 \\
      & PopQA           & 16.23 & 15.83 & 16.16 & 16.44 & 15.09 & 15.57 \\
    \midrule
    \multicolumn{2}{l}{\textbf{Overall Average (14 tasks)}}
      & \textbf{35.73} & \textbf{36.36} & \textbf{36.34} & \textbf{35.22} & \textbf{35.38} & \textbf{35.57} \\
    \bottomrule
  \end{tabular}
  \caption{Per-benchmark scores at step 78{,}000 (base models, no EMA).}
  \label{tab:benchmark_results}
\end{table}

\section{Stability across Pretraining Runs}
\label{app:stability}

Two of the configurations were pretrained a second time under the same recipe, and the comparison survives it. Over the fourteen benchmarks of the reported scope, ECHO MTP-1 leads DeepSeek's MTP-1 by 0.52 points in the baseline run and by 0.78 in the repeat. Reasoning ($+0.57$ then $+1.33$) and code ($+1.35$ then $+1.69$) both grow, multi-subject multiple choice turns in ECHO's favor ($-0.52$ then $+1.30$), closed-book QA stays within a fifth of a point of zero in both, and mathematics is the one category that moves the other way ($+0.79$ then $-0.73$). Two evaluations of the same weights differ by 0.27 points per benchmark, which puts the noise at 0.13 on a fourteen-benchmark gap and 0.30 on a category gap. \Cref{tab:train-twice} gives the three column groups, and \Cref{fig:s5-conditions} draws the two classes with the spread these two runs measure.

\paragraph{Twice the training budget.} The baseline run was then carried to step $156{,}000$ on both sides, twice its token budget. Over the fourteen rows ECHO's lead holds, $+0.52$ at step 78k against $+0.50$ at step 156k, a change of $-0.02$ against a spread of $0.14$ across rows, with seven rows in its favor at each step, and every bucket is reproduced to within $0.2$ points. Doubling the budget therefore leaves both the overall gap and the category ordering where they were. Repeating the run reproduces the overall gap to $0.26$ points but not its split across categories, which moves by $0.3$ to $1.8$ points.

\begin{table}[!htbp]
  \centering
  \footnotesize
  \setlength{\tabcolsep}{6pt}
  \begin{tabular}{l rr rr rr}
    \toprule
    & \multicolumn{2}{c}{Baseline} & \multicolumn{2}{c}{Repeat run} & \multicolumn{2}{c}{Twice the budget} \\
    & \multicolumn{2}{c}{run 1, step 78k} & \multicolumn{2}{c}{run 2, step 78k} & \multicolumn{2}{c}{run 1, step 156k} \\
    \cmidrule(lr){2-3}\cmidrule(lr){4-5}\cmidrule(lr){6-7}
    Benchmark & ECHO & DeepSeek & ECHO & DeepSeek & ECHO & DeepSeek \\
    \midrule
    GSM8K         & 61.49 & 61.64 & 60.05 & 60.73 & 64.82 & 65.05 \\
    MATH          & 36.04 & 34.31 & 34.58 & 35.37 & 38.49 & 36.65 \\
    \midrule
    BBH           & 46.44 & 46.91 & 48.15 & 47.79 & 48.31 & 48.79 \\
    BBEH          & 6.63 & 5.33 & 5.83 & 5.59 & 6.18 & 5.29 \\
    PIQA          & 77.57 & 77.68 & 78.67 & 76.64 & 79.71 & 79.59 \\
    WinoGrande    & 63.85 & 62.27 & 64.96 & 62.27 & 65.67 & 63.30 \\
    \midrule
    MMLU          & 55.93 & 56.39 & 56.08 & 56.30 & 59.23 & 59.36 \\
    MMLU-Pro      & 26.02 & 26.60 & 27.04 & 24.21 & 26.23 & 27.28 \\
    \midrule
    HumanEval+    & 26.22 & 23.17 & 28.05 & 25.00 & 29.88 & 26.83 \\
    MBPP+         & 47.62 & 47.62 & 49.21 & 48.68 & 50.00 & 50.26 \\
    LiveCodeBench & 7.00 & 6.00 & 6.75 & 5.25 & 7.25 & 6.50 \\
    \midrule
    PreciseWikiQA & 7.88 & 7.48 & 8.12 & 8.00 & 8.56 & 9.18 \\
    HotpotQA      & 21.35 & 21.18 & 22.11 & 22.10 & 23.23 & 21.90 \\
    PopQA         & 16.23 & 16.44 & 15.48 & 16.20 & 16.38 & 16.97 \\
    \bottomrule
  \end{tabular}
  \caption{One recipe under two perturbations, both against the same baseline. \textbf{Baseline}
    is the run reported in \Cref{tab:benchmark_results}. \textbf{Repeat run} is an independent
    pretraining run of the same recipe at the same step. \textbf{Twice the budget} is the
    baseline run continued to step~156k.}
  \label{tab:train-twice}
\end{table}

\section{Next-Token Loss Does Not Separate the Two Objectives}
\label{app:trainloss}

\Cref{fig:trainloss} separates the objective into its two terms. On the backbone's next-token
loss the two runs are indistinguishable: the gap averages $0.004$ over the 78k steps and stays
between $0.004$ and $0.005$ in every 10k-step window, with ECHO-1 marginally above MTP-1
throughout. The MTP loss carries a larger gap in the same direction, $0.011$ at step~78k.
On the same token budget the two objectives therefore arrive at the same next-token likelihood,
and neither margin is one the method is meant to win.

\begin{figure}[!htbp]
    \centering
    \includegraphics[width=\textwidth]{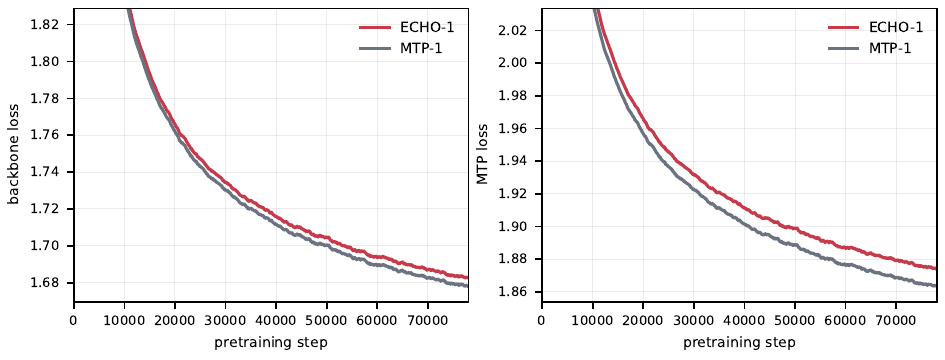}
  \caption{Training loss of ECHO-1 and MTP-1 over the first 78k steps of the baseline run,
    separated into the backbone's next-token loss (left) and the MTP loss (right).
    The logged objective is the weighted sum
    $\mathcal{L}=\mathcal{L}_{\mathrm{backbone}}+0.3\,\mathcal{L}_{\mathrm{MTP}}$, and the left
    panel is recovered by subtracting the weighted MTP loss from it. Each curve is an
    exponential moving average over all $78{,}000$ logged steps, with no subsampling; the
    vertical range is clipped to what the curves occupy after step 12k, so the first few
    thousand steps run above the frame.}
  \label{fig:trainloss}
\end{figure}

This is the expected outcome rather than a negative result. ECHO's multi-horizon objective adds no
capacity to the next-token head; it changes what the shared representation is required to support.
In predicting the next token the model is held responsible for several future positions at once,
which constrains the hidden state towards features that stay predictive beyond one step. The effect
of such a constraint is regularization, and regularization does not surface in the quantity it
constrains: next-token cross-entropy is the term both objectives already optimize, so matching it
is the premise of the comparison rather than its outcome. The two configurations separate
downstream instead, in the benchmark gaps of \Cref{tab:benchmark_results,tab:train-twice} and in
the draft quality that \Cref{sec:serving} converts into decode speedup.

\begin{figure}[!htbp]
    \centering
    \includegraphics[width=\textwidth]{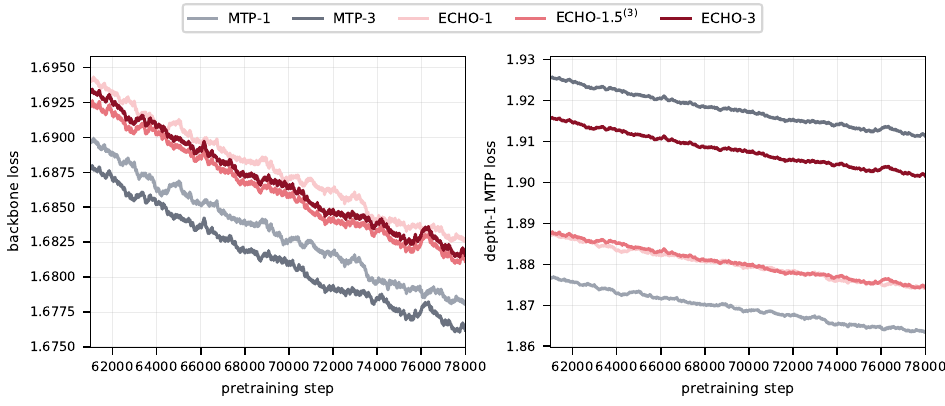}
  \caption{The same two terms across both families, over the steps all five runs share.
    ECHO-1.5$^{(3)}$'s event files start at step~60k on our side and the first thousand of
    those go to the moving average's warm-up, which fixes the window at 61k--78k; the
    full-run view of the depth-1 pair is \Cref{fig:trainloss}. On the backbone term (left) the
    five lie within $0.006$ and the ordering follows the family rather than the depth; the
    depth-1 term (right) is where added depth shows up.}
  \label{fig:trainloss-fam}
\end{figure}
\FloatBarrier

\paragraph{The same holds along depth.} \Cref{fig:trainloss-fam} repeats the split for all
five configurations, ECHO-1, ECHO-1.5$^{(3)}$ and ECHO-3 against MTP-1 and MTP-3, over the
steps all the runs share. Within a family the backbone term barely moves with depth: the three ECHO
configurations end within $0.002$ of each other, and MTP-3 sits $0.0017$ below MTP-1. Across all
five it spans $0.006$, and the ordering is not ours to like --- MTP-3 holds the lowest backbone
loss of any configuration, $1.676$ against ECHO-3's $1.682$, while ECHO-3 leads it by $0.97$
points over the fourteen benchmarks of \Cref{tab:benchmark_results}. What does move is the MTP
term, and it moves the way the weighting predicts: the per-step loss is averaged over the depths
trained at that step, so adding depths takes weight off depth~1, and the depth-1 term rises by
$0.027$ from ECHO-1.5$^{(3)}$ to ECHO-3 and by $0.048$ from MTP-1 to MTP-3. At matched depth~3
the two designs divide the horizons: reusing the final layer fits the nearest one better
($1.902$ against $1.911$) and the separate modules fit the far ones better ($1.992$ and $2.045$
against $2.002$ and $2.083$ at depths~2 and~3). ECHO-1.5$^{(3)}$'s deeper terms are logged only
on the steps that sample them, $39.5\%$ and $9.8\%$ of steps for depths~2 and~3, which is the
$f_k = \mathbb{P}(D\ge k) = (1, 0.4, 0.1)$ that $\pi=(0.6,0.3,0.1)$ prescribes in
\Cref{eq:dds-weight}. At matched step those sampled depths are where the saving is paid for:
\Cref{fig:depth2-step} puts ECHO-1.5$^{(3)}$'s depth-2 term $0.051$ above ECHO-3's at step~78k
and its depth-3 term $0.132$ above, with the curves running parallel rather than converging.
That is the shape of the trade, and it stops there. Depth~1, which both train at every step,
costs nothing at all ($1.874$ against $1.902$, in ECHO-1.5$^{(3)}$'s favour); the fourteen-task
average is within $0.02$ ($36.34$ against $36.36$, \Cref{tab:benchmark_results}); and the
accepted length holds beyond the trained depth (\Cref{fig:extra}). None of these margins is
where the comparison is settled, which is this
section's point: the terms of the training objective do not rank the two designs.

\section{After Supervised Finetuning}
\label{app:sft}

The pretrained checkpoints were finetuned on the chat SFT corpus, and
\Cref{tab:sft} reports the fourteen in-scope benchmarks for the two ECHO
configurations against DeepSeek's MTP baseline finetuned the same way.
ECHO keeps its lead after
finetuning, $+1.61$ and $+2.10$ points over the fourteen rows, and the
lead is again larger on reasoning and code than on QA and multiple choice.
Three caveats apply. The cells are decoded greedily, so
repeating an evaluation does not move the rule- and execution-graded rows, and no
error bars are drawn for them. The four judge-graded rows, BBEH and the three
closed-book sets, do carry grader variability.
BBH is read in its chain-of-thought form: under the answer-directly prompt the
finetuned output degenerates to a constant option label. 

\begin{table}[!htbp]
  \centering
  \footnotesize
  \setlength{\tabcolsep}{4pt}
  \begin{tabular}{llrrr}
    \toprule
    Category & Benchmark & MTP-Instruct & ECHO-1.5$^{(3)}$-Instruct & ECHO-2$^{(7)}$-Instruct \\
    \midrule
    \multirow{2}{*}{Mathematics}
      & GSM8K         & 61.87 & 62.70 & 65.20 \\
      & MATH          & 31.05 & 38.54 & 38.27 \\
    \midrule
    \multirow{4}{*}{Reasoning}
      & BBH           & 54.41 & 54.24 & 58.06 \\
      & BBEH          & 11.30 & 10.33 & 10.00 \\
      & PIQA          & 74.49 & 74.38 & 74.38 \\
      & WinoGrande    & 61.17 & 61.09 & 61.56 \\
    \midrule
    \multirow{2}{*}{Multi-subject MC}
      & MMLU          & 49.17 & 52.46 & 55.53 \\
      & MMLU-Pro      & 29.06 & 32.25 & 31.28 \\
    \midrule
    \multirow{3}{*}{Code}
      & HumanEval+    & 47.56 & 48.78 & 51.83 \\
      & MBPP+         & 44.97 & 51.59 & 46.30 \\
      & LiveCodeBench & 8.75 & 11.25 & 9.75 \\
    \midrule
    \multirow{3}{*}{Closed-book QA}
      & PreciseWikiQA & 4.22 & 4.90 & 5.06 \\
      & HotpotQA      & 19.04 & 16.82 & 19.16 \\
      & PopQA         & 10.69 & 10.94 & 10.77 \\
    \midrule
    \multicolumn{2}{l}{\textbf{Overall Average (14 tasks)}}
      & \textbf{36.27} & \textbf{37.88} & \textbf{38.37} \\
    \bottomrule
  \end{tabular}
  \caption{The fourteen in-scope benchmarks after one round of supervised
    finetuning. Every cell is scored under the protocol of
    \Cref{tab:benchmark_results}.}
  \label{tab:sft}
\end{table}

\section{Decode Speedup under Serving}

\Cref{fig:entropy} reports the cells of \Cref{fig:accept_entropy_merged} as decode speedup instead of accepted length; the ordering of the drafters is the one discussed in \Cref{sec:serving}.

\section{Comparability of Reported Baselines in Inference}
\label{sec:lack_fair_comparison}

The evaluation scales differ systematically. EAGLE-3 reports results across a wide range of target sizes, from Vicuna-13B up to LLaMA-3.3-70B, whereas the parallel drafters cluster at the small end: Domino~\citep{huang2026dominodecouplingcausalmodeling} evaluates only Qwen3-4B and Qwen3-8B, DSpark extends to Qwen3-14B, and the largest target any of them reports is Qwen3-Coder-30B-A3B, a sparse model with roughly 3B active parameters, so its per-token cost is closer to a 4B dense model than to a 30B one. No paper in this family evaluates a dense target above 14B (\Cref{tab:eval-models}).

\newcommand{\yes}{\textcolor{ForestGreen}{\ding{51}}}
\newcommand{\no}{\textcolor{BrickRed}{\ding{55}}}

\begin{table}[!htbp]
\centering
\caption{Target models evaluated by each method, ordered by model size.}
\label{tab:eval-models}
\setlength{\tabcolsep}{4pt}
\resizebox{\textwidth}{!}{%
\begin{tabular}{lcccccccc}
\toprule
Model & Qwen3-4B & LLaMA-3.1-8B  & Qwen3-8B & Gemma4-12B & Vicuna-13B & Qwen3-14B & Qwen3-Coder-30B-A3B & LLaMA-3.3-70B \\
\midrule
EAGLE-3 & \no  & \yes & \no  & \no  & \yes & \no  & \no  & \yes \\
DFlash  & \yes & \yes & \yes & \no  & \no  & \no  & \yes & \no  \\
Domino  & \yes & \no  & \yes & \no  & \no  & \no  & \no  & \no  \\
DSpark  & \yes & \no  & \yes & \yes & \no  & \yes & \no  & \no  \\
\bottomrule
\end{tabular}%
}
\end{table}

Beyond scale, the EAGLE-3 checkpoints these papers evaluate were never fairly trained (\Cref{tab:unfair-comparison-family}). DFlash and Domino used public checkpoints released by AngelSlim with the different training datasets and fewer samples. Drafter capacity is uncontrolled as well: at every target, the block-parallel drafter carries $2.1$--$2.6\times$ the parameters of the EAGLE-3 head. 

\begin{table}[!htbp]
\centering
\caption{Training data and drafter capacity in the DFlash-family papers' comparisons against EAGLE-3. Drafter parameter counts are computed from released configs.}
\label{tab:unfair-comparison-family}
\setlength{\tabcolsep}{5pt}
\resizebox{\textwidth}{!}{%
\begin{tabular}{lllrr}
\toprule
Method & Model & Training datasets & Samples & Draft params \\
\midrule
DFlash  & Qwen3-4B / 8B & Nemotron V2, CodeAlpaca, ShareGPT, UltraChat, LongAlign & 800K & 459M \\
EAGLE-3 & Qwen3-4B / 8B & ShareGPT, UltraChat & 530K & 218M \\
\midrule
Domino  & Qwen3-8B & Open-PerfectBlend & 1.42M & 1049M \\
EAGLE-3 & Qwen3-8B & ShareGPT, UltraChat & 530K & 400M \\
\midrule
DSpark  & Qwen3-4B / 8B / 14B & Open-PerfectBlend & 1.3M & 459M / 1049M / 1573M \\
EAGLE-3 & Qwen3-4B / 8B / 14B & Open-PerfectBlend & 1.3M & 218M / 400M / 609M \\
\bottomrule
\end{tabular}%
}
\end{table}
\FloatBarrier

Although DSpark controls the data axis: it retrains EAGLE-3 and DSpark within a single framework on identical training dataset, however, its evaluation restricts EAGLE-3 to chain drafting instead of the original tree drafting. The released code proposes one token per step and returns a single trajectory with no tree construction, top-k expansion, or candidate reranking. The baseline is therefore matched in data but evaluated in a degraded configuration, and its acceptance length understates what the method achieves as published.

\begin{figure}[!htbp]
    \centering
    \includegraphics[width=3.3in]{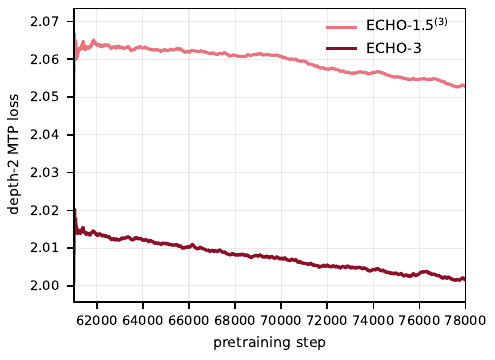}
  \caption{The depth-2 term at matched pretraining step. ECHO-1.5$^{(3)}$ samples depth~2 on
    $39.5\%$ of steps where ECHO-3 trains it on all of them, so the sampled series carries fewer
    points over this window, $6{,}721$ against $17{,}000$. Fewer is not biased: at depth~1,
    which both configurations train every step, the loss on the steps that also sample depth~2
    differs from the loss on the steps that do not by $+0.0002$, against a standard error of
    $0.0004$.}
  \label{fig:depth2-step}
\end{figure}

\section{Artifact Licensing}

\Cref{tab:artifact_licenses} summarizes the
licenses and usage terms of the training and
evaluation datasets used in our experiments.
The pretraining corpus and in-house SFT dataset
are proprietary and subject to restricted
redistribution. All fourteen evaluation benchmarks
are used solely for research evaluation and are
not redistributed.
The use of these artifacts follows their applicable
licenses and access conditions.

\begin{figure}[!htbp]
    \centering
    \includegraphics[width=3.3in]{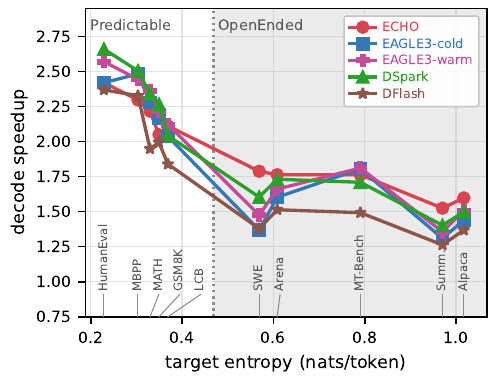}
  \caption{Decode speedup at $T{=}1$ with rejection sampling, over the same cells as
    \Cref{fig:accept_entropy_merged}. The
    warm-start EAGLE-3 drafter is trained from DeepSeek's MTP baseline, following
    \citet{kimiteam2026k3}, and is normalized by the baseline of its own run.}
  \label{fig:entropy}
\end{figure}
\FloatBarrier

\begin{table}[!htbp]
\centering
\footnotesize
\begin{tabular}{lll}
\toprule
Type & Artifact & License / Terms of Use \\
\midrule

\multicolumn{3}{l}{\textit{Training Data}} \\
\midrule
Pretraining & 600B-token corpus
    & Proprietary (redistribution restricted) \\
SFT & In-house chat corpus
    & Proprietary (redistribution restricted) \\

\midrule
\multicolumn{3}{l}{\textit{Model Weights and Checkpoints}} \\
\midrule
Weights & DeepSeek-V3 pretrained weights
    & DeepSeek Model License \\
Model & gpt-oss-120b
    & Apache-2.0 \\

\midrule
\multicolumn{3}{l}{\textit{Evaluation Datasets}} \\
\midrule
Mathematics & MATH
    & MIT (repository) \\
Mathematics & GSM8K
    & MIT \\
Reasoning & BBH
    & MIT \\
Reasoning & BBEH
    & CC BY 4.0 (data) \\
Reasoning & WinoGrande
    & Apache-2.0 \\
Reasoning & PIQA
    & Not specified \\
Multi-subject & MMLU
    & MIT \\
Multi-subject & MMLU-Pro
    & MIT \\
Code & HumanEval+
    & Apache-2.0 (EvalPlus) \\
Code & MBPP+
    & Apache-2.0 (EvalPlus) \\
Code & LiveCodeBench
    & MIT (code); data terms unclear \\
QA & PopQA
    & MIT (repository) \\
QA & PreciseWikiQA
    & CC BY-NC 4.0 (HalluLens) \\
QA & HotpotQA
    & CC BY-SA 4.0 \\

\midrule
\multicolumn{3}{l}{\textit{Additional Serving Benchmarks}} \\
\midrule
Code & HumanEval
    & MIT \\
Code & MBPP
    & CC BY 4.0 \\
Mathematics & MATH-500
    & MIT \\
Software & SWE-bench
    & MIT (repository) \\
Chat & Arena-Hard
    & Apache-2.0 (repository) \\
Chat & MT-Bench
    & Not confirmed (prompt dataset) \\
Summarization & Summarization prompts
    & Not specified \\
Instruction & Stanford Alpaca
    & CC BY-NC 4.0 \\

\midrule
\multicolumn{3}{l}{\textit{Software and Implementations}} \\
\midrule
Serving & SGLang v0.4
    & Apache-2.0 \\
Attention & FlashAttention-3
    & BSD-3-Clause \\
Drafter & EAGLE-3
    & Apache-2.0 \\
Drafter & DFlash
    & MIT \\
Drafter & DSpark
    & MIT  \\
Evaluation & EvalPlus
    & Apache-2.0 (repository) \\

\bottomrule
\end{tabular}
\caption{Licenses and terms of use for the scientific
artifacts used in our experiments, including training
data, pretrained models, evaluation benchmarks, and
third-party software. Repository-level licenses are
distinguished from dataset-specific and model-weight
licenses where applicable.}
\label{tab:artifact_licenses}
\end{table}

\section{AI Use Disclosure}

During the preparation of this manuscript, the authors used ChatGPT (OpenAI) to assist with language editing, grammatical corrections, and improving the clarity of author-written text. The AI tool was not used to generate research findings or experimental results. All AI-assisted modifications were reviewed and verified by the authors, who take full responsibility for the content of this manuscript.

\end{document}